\documentclass{article} 
\usepackage{preprint,times}

\usepackage{amsmath,amsfonts,bm}

\def\eqref#1{equation~\ref{#1}}

\def\1{\bm{1}}

\def\mA{{\bm{A}}}
\def\mB{{\bm{B}}}
\def\mC{{\bm{C}}}
\def\mD{{\bm{D}}}

\def\mG{{\bm{G}}}

\def\mM{{\bm{M}}}

\def\mU{{\bm{U}}}
\def\mV{{\bm{V}}}

\def\mX{{\bm{X}}}
\def\mY{{\bm{Y}}}

\DeclareMathAlphabet{\mathsfit}{\encodingdefault}{\sfdefault}{m}{sl}
\SetMathAlphabet{\mathsfit}{bold}{\encodingdefault}{\sfdefault}{bx}{n}

\DeclareMathOperator*{\argmin}{arg\,min}

\usepackage[utf8]{inputenc} 
\usepackage[T1]{fontenc}    
\usepackage{hyperref}       
\usepackage{url}            
\usepackage{booktabs}       
\usepackage{amsfonts}       
\usepackage{nicefrac}       
\usepackage{microtype}      
\usepackage{xcolor}         
\usepackage{amsthm, amssymb, amsfonts, latexsym, mathtools}
\usepackage{algorithm}
\usepackage{algpseudocode}
\usepackage{comment}
\usepackage{here}
\usepackage{caption}
\usepackage{subcaption}
\usepackage{wrapfig}
\usepackage{multirow}
\usepackage{color}
\usepackage{wrapfig}
\usepackage{multirow}
\usepackage{siunitx}
\usepackage{bbm}
\usepackage{enumitem}
\usepackage[font=small]{caption}

\usepackage[capitalise]{cleveref}

\newtheorem{theorem}{Theorem}

\newtheorem{lemma}{Lemma}

\newtheorem{remark}{Remark}

\newtheorem{assumption}{Assumption}
\newtheorem{example}{Example}

\crefname{theorem}{Theorem}{Theorems}
\crefname{lemma}{Lemma}{Lemmas}
\crefname{equation}{Eq.~}{Eqs.~}
\crefname{remark}{Remark}{Remarks}
\crefname{assumption}{Assumption}{Assumptions}
\crefname{algorithm}{Algorithm}{Algorithms}

\newcommand{\proposed}{\textsc{FedMuon}}

\title{FedMuon: Federated Learning with Bias-corrected LMO-based Optimization}

\iclrfinalcopy

\author{Yuki Takezawa \\
Kyoto University, OIST
\And
Anastasia Koloskova \\
University of Zurich
\And
Xiaowen Jiang \\
CISPA, Saarland University
\And
Sebastian U. Stich \\
CISPA
}

\begin{document}

\maketitle

\begin{abstract}
Recently, a new optimization method based on the linear minimization oracle (LMO), called Muon, has been attracting increasing attention since it can train neural networks faster than existing adaptive optimization methods, such as Adam.
In this paper, we study how Muon can be utilized in federated learning.
We first show that straightforwardly using Muon as the local optimizer of FedAvg does not converge to the stationary point since the LMO is a biased operator.
We then propose \proposed{} that can mitigate this issue. 
We also analyze how solving the LMO approximately affects the convergence rate and find that, surprisingly, \proposed{} can converge for any number of Newton-Schulz iterations, while it can converge faster as we solve the LMO more accurately.
Through experiments, we demonstrated that \proposed{} can outperform the state-of-the-art federated learning methods.
\end{abstract}

\section{Introduction}

Federated learning, which can train neural networks in parallel across many clients, has been attracting much attention \citep{kariouz2021advances,mcmahan2017communication,karimireddy2020scaffold}.
In federated learning, each client has its own training datasets and updates its parameters using a local optimizer, such as SGD. The central server collects the parameters from the clients and aggregates them.
Since clients do not need to share their local training datasets with others, federated learning inherently preserves data privacy.

For training neural networks efficiently, using an appropriate stepsize is one of the most critical factors.
If the stepsize is too large, the training collapses, whereas if the stepsize is too small, the training requires a huge number of iterations.
To adjust the stepsize on the fly during the training, using adaptive optimization methods, such as AdaGrad \citep{duchi2011adaptive}, Adam \citep{kingma2017adam}, Shampoo \citep{gupta2018shampoo}, and other methods \citep{loshchilov2018decoupled,vyas2025soap}, have long been regarded as the de facto standard for training neural networks.

Recently, Muon \citep{liu2025muon} has emerged as a promising alternative, attracting significant attention.
Many papers evaluated the performance of Muon and demonstrated that Muon can train neural networks faster and achieve higher accuracy than the existing optimization methods, such as AdamW \citep{liu2025muon,semenov2025benchmarking}.
Roughly speaking, Muon projects the momentum in the Momentum SGD onto the space of orthogonal matrices.
Muon is closely related to various optimization methods: it can be interpreted as a simplified version of Shampoo \citep{gupta2018shampoo}, in which a certain momentum accumulation is disabled \citep{liu2025muon}, and as an instance of optimizers with linear minimization oracle (LMO) under a specific norm \citep{pethick2025training}.
\citet{kovalev2025understanding} also showed that Muon is a special instance of the trust-region optimization method.

To use Muon for the large-scale training, developing the distributed version of Muon is important.
However, Muon requires us to solve the LMO every iteration, which makes it difficult to straightforwardly use Muon in a distributed environment.
\citet{ahn2025dion} proposed a method to solve the LMO in a distributed manner, although their method does not support multiple local steps and incurs a huge communication cost.
\citet{therien2025muloco} proposed MuLoCo, which extends Muon by allowing clients to update the parameters multiple times by Muon as in Local SGD \citep{stich2019local,woodworth2020is}.  
Although \citet{therien2025muloco} demonstrated that MuLoCo performs well when all clients share the same training dataset, their method is limited to homogeneous settings and lacks theoretical guarantees.
As we show in \cref{sec:local_muon}, MuLoCo fails to converge when clients have different datasets, which is a fundamental characteristic of federated learning.

In this paper, we study the federated learning methods with the LMO and propose \proposed{}.
(i) First, we show that straightforwardly using Muon as the local optimizer in FedAvg failed to converge to the stationary point since the LMO is a biased operator.
We formally analyze the lower bound of this straightforward method, showing that it does not converge to the stationary point, especially in the heterogeneous setting.
(ii) We then propose \proposed{}, which can mitigate the bias caused by the LMO and can provably converge to the stationary point.
(iii) Furthermore, we derive a novel analysis and reveal how the inexact LMO affects the convergence behavior of \proposed{}.
Since solving the LMO exactly is computationally expensive, we solve the LMO approximately by running the Newton-Schulz iteration \citep{schulz1933iterative} several times in practice.
There were many papers that analyzed the convergence behavior of Muon \citep{riabinin2025gluonmakingmuon,liu2025muon,shen2025convergence}, while most of them assumed that the LMO is solved exactly and ignored the effect caused by the inexact LMO.
We analyze the impact of inexact solutions to the LMO on the convergence rate.
We discover that for any number of Newton-Schulz iterations, \proposed{} can converge to the stationary point and can converge faster by up to a factor proportional to the square root of the dimension of the parameters as we solve the LMO more accurately.
We experimentally demonstrated the effectiveness of \proposed{}, showing that \proposed{} can achieve higher accuracy than the state-of-the-art adaptive federated learning optimization methods.

Our contributions are summarized as follows:
\begin{itemize}[nosep,leftmargin=12pt]
    \item We show that directly plugging Muon into FedAvg as the local optimizer does not converge to the stationary point since the LMO is a biased operator.
    \item We propose \proposed{}, which mitigates the above issue by the bias correction mechanism and can converge to the stationary point.
    \item We analyze the convergence rate of \proposed{} with the inexact LMO. Then, we show that for any number of Newton-Schulz iterations, \proposed{} can converge, revealing how the Newton-Schulz iteration affects the convergence rate.
    \item Through the experiments, we demonstrated that \proposed{} can outperform the state-of-the-art federated learning methods.
\end{itemize}

\paragraph{Notation:}
We use $\| \cdot \|$ to denote an arbitrary norm, and its dual norm is denoted by $\| \cdot \|_\star$. 
When we refer to a specific norm, we explicitly use the notation $\| \cdot \|_p$, $\| \cdot \|_F$, $\| \cdot \|_\text{sp}$, and $\| \cdot \|_\text{trace}$ to denote the Schatten $p$-norm, Frobenius norm, spectral norm, and trace norm, respectively.
We denote $[n] = \{ 1, 2, \dots, n \}$ for any $n \in \mathbb{N}$.

\section{Preliminary}
\label{sec:preliminary}
In this section, we briefly introduce federated learning and Muon.
The detailed discussion about the related works is deferred to \cref{sec:related_work}.

\paragraph{Federated Learning:}
We consider the following problem where the loss functions are distributed among $n$ clients:
\begin{align*}
    \min_{\mX \in \mathcal{X}} \left[ f (\mX) \coloneqq \frac{1}{n} \sum_{i=1}^n f_i (\mX) \right],
    \qquad f_i (\mX) \coloneqq \mathbb{E}_{\xi_i \sim \mathcal{D}_i} [ F_i (\mX ; \xi_i) ],
\end{align*}
where $\mathcal{X}$ is the parameter space (e.g., $\mathbb{R}^d$ or $\mathbb{R}^{d_1 \times d_2}$), $\mX$ is the model parameter, $\mathcal{D}_i$ is the training that client $i$ holds, and $f_i : \mathcal{X} \rightarrow \mathbb{R}$ is the loss function of client $i$.

The most fundamental algorithm for federated learning is Federated Averaging (FedAvg) \citep{mcmahan2017communication}.
In FedAvg, each client updates the parameter by using its own loss function, and then the central server aggregates the parameters sent from the clients.
The update rule of FedAvg is described in \cref{sec:algorithm_of_local_muon}.
The original FedAvg uses SGD as the local optimizer, while, as in the non-distributed learning, it is important to use adaptive optimization methods for stable and fast training.
Many papers proposed federated learning methods that use more sophisticated optimizers, such as Momentum SGD \citep{lin2021quasi}, Adam \citep{reddi2021adaptive}, and the Newton method \citep{elgabli2022fednew}.
\citet{reddi2021adaptive} proposed a general framework and analyzed the convergence rate with various optimizers.

\paragraph{Optimizer with Linear Minimization Oracle:}
Recently, optimizers with linear minimization oracle (LMO) have been attracting a lot of attention \citep{liu2025muon,pethick2025training,riabinin2025gluonmakingmuon}.
LMO is defined as follows:
\begin{align*}
    \text{lmo} (\mX ; \mathcal{D}) &\coloneqq \argmin_{\mY \in \mathcal{D}} \langle \mX, \mY \rangle,
\end{align*}
where $\mathcal{D}$ is the convex set and $\langle \mX, \mY \rangle \coloneqq \sum_{i,j} X_{ij} Y_{ij}$.
Originally, the LMO has been used to solve the convex constrained problems in the Frank-Wolfe algorithm \citep{frank1956algorithm,jaggi2013revisiting}.
Recently, \citet{jordan2024muon} proposed Muon, which uses LMO for training neural networks,
which is an unconstrained optimization problem. They showed that Muon can train neural networks faster than AdamW \citep{loshchilov2018decoupled} and Shampoo \citep{gupta2018shampoo,shi2023distributed}, which are the most commonly used optimizers these days.

Specifically, the optimizers with LMO choose the unit ball as the constraint set $\mathcal{D}$, measured in any chosen norm $\|\cdot\|$.
With a slight abuse of notation, we use $\text{lmo} (\cdot)$ to represent $\text{lmo} (\cdot ; \mathcal{D})$ with $\mathcal{D} \coloneqq \{ \mY \in \mathcal{X} \mid \| \mY \| \leq 1\}$, i.e., 
\begin{align*}
    \text{lmo} (\mX) &\coloneqq \argmin_{\mY \in \{ \mY \in \mathcal{X} \mid \| \mY \| \leq 1\}} \langle \mX, \mY \rangle.
\end{align*}
Then, the update rules are given by:
\begin{align*}
    \mM^{(r+1)} &= (1 - \alpha) \mM^{(r)} + \alpha \nabla F (\mX^{(r)} ; \xi^{(r)}), \\
    \mX^{(r+1)} &= \mX^{(r)} + \eta \text{lmo} (\mM^{(r+1))}).
\end{align*}
By varying the norm, we can recover different popular optimizers.
Specifically, if parameter space is a vector when we choose the Euclidean norm and max norm, we can recover Normalized SGD with momentum \citep{cutkosky2020momentum} and Sign SGD with momentum \citep{sun2023momentum}, respectively.
Then, if the parameter space is $\mathbb{R}^{d_1 \times d_2}$ and we use the spectral norm for the LMO, we can obtain Muon \citep{jordan2024muon}.
Note that the parameter space needs to be the space of $d_1 \times d_2$ matrices for Muon.
Each layer is taken into account separately.
For instance, the parameter of the convolutional layer is $\textit{out\_channel} \times \textit{in\_channel} \times h \times w$ matrix.
When we use Muon, we consider $d_1 = \textit{out\_channel}$ and $ d_2= \textit{in\_channel} \times h \times w$.
The remaining scalar and vector parameters in the neural network are trained by other optimization methods, such as SGD or Adam.

For the remainder of the paper, we will not take separate layers into account and represent all the model parameters as a single matrix for simplicity of presentation. We refer to \cite{riabinin2025gluonmakingmuon} for an explanation of how to take into account every layer separately in the analysis of Muon.

\section{LocalMuon does not Always Converge}
\label{sec:local_muon}
First, we provide a lower bound showing that straightforwardly using the optimizer with the LMO as the local optimizer in FedAvg does not always converge.
For simplicity, we consider the setting where all clients participate in every round and perform exactly one local update.
Straightforwardly applying the optimizer with the LMO to FedAvg yields the following update rules:
\begin{align}
    \label{eq:simple_rule_1}
    \mM_i^{(r+1)} &= (1 - \alpha) \mM_i^{(r)}  + \alpha \nabla F_i (\mX^{(r)}, \xi_i^{(r)}), \\
    \label{eq:simple_rule_2}
    \mX_i^{(r+1)} &= \mX_i^{(r)} + \eta \text{lmo} \left( \mM_i^{(r+1)} \right), \\
    \label{eq:simple_rule_3}
    \mX^{(r+1)} &= \frac{1}{n} \sum_{i=1}^n \mX_i^{(r+1)}.
\end{align}
We refer to the above algorithm as \textsc{LocalMuon} (see \cref{sec:algorithm_of_local_muon} for \textsc{LocalMuon} with multiple local steps and partial participation).
However, the above straightforward method fails to reach a stationary point due to the bias introduced by the LMO, and the optimization process stagnates.
Specifically, the LMO is biased, since in general we have\begin{align*}
    \frac{1}{n} \sum_{i=1}^n \text{lmo} \left( \mM_i^{(r+1)} \right) \not = \text{lmo} \left( \frac{1}{n} \sum_{i=1}^n \mM_i^{(r+1)} \right).
\end{align*}
The momentum $\mM_i$ is the estimation of the gradient $\nabla f_i (\mX)$, while the quantity of $\tfrac{1}{n} \sum_{i=1}^n \text{lmo} (\mM_i)$ is biased and does not align with the gradient $\nabla f (\mX)$.
This intuitively shows why \textsc{LocalMuon} cannot converge to the stationary point, especially when clients have different loss functions.
The following theorem formalizes this failure, with the proof deferred to \cref{sec:proof_of_lower_bound}.

\begin{theorem}
\label{theorem:lower_bound}
For simplicity, we consider the initialization $\mM_i^{(0)} = 0$.
There exist convex functions $\{ f_i \}_{i=1}^n$ such that for any $r \geq 1$ rounds, the output of \textsc{LocalMuon} (\cref{eq:simple_rule_1,eq:simple_rule_2,eq:simple_rule_3}) is the same as the initial parameter and does not converge to the optimal solution and satisfies the following:
\begin{align*}
    \| \nabla f (\mX^{(r)}) \|^2 \geq \Omega(\zeta^2_\star), 
\end{align*}
where $\zeta^2_\star \coloneqq \frac{1}{n} \sum_{i=1}^n \left\| \nabla f_i (\mX^\star) \right\|^2$ and $\mX^\star \coloneqq \argmin f(\mX)$.
\end{theorem}

Note that \textsc{LocalMuon} is a simplified version of MuLoCo \citep{therien2025muloco}, where the momentum at the central server is disabled.
We formally analyze only \textsc{LocalMuon}, while \cref{theorem:lower_bound} shows that the parameter stays at the initial parameters and does not converge to the stationary point.
This indicates that MuLoCo also suffers from the same issue, as adding the momentum at the central server does not prevent the parameter from remaining at its initial value.

\begin{algorithm}[b!]
\renewcommand\algorithmicindent{1.2em}
\caption{\proposed{}}
\label{algorithm:proposed_method}
\begin{algorithmic}[1]
\State \textbf{Input:} total number of clients $n$, number of sampled clients $S$, and the number of local steps $K$.
\For{$r \in \{ 0, 1, \cdots, R-1 \}$}  \emph {\color{gray} (at the server)}
    \State sample $S$ clients $\mathcal{S}_{r} \subset [n]$. 
    \State send $\mX^{(r)}$ and $\mC^{(r)}$ to the sampled clients.
    \For{$i \in \mathcal{S}_r$}   \emph {\color{gray} (at the clients)}
        \State $\mX_i^{(r, 0)} \leftarrow \mX^{(r)}$ and $\mM_i^{(r, 0)} \leftarrow \mM_i^{(r-1, K)}$
        \For{$k = 0, 1, \cdots, K-1$}
            \State $\mM_i^{(r, k+1)} \leftarrow (1 - \alpha) \mM_i^{(r, k)} + \alpha \nabla F_i (\mX_i^{(r, k)} ; \xi_i^{(r, k)})$.
            \State $\mX_i^{(r, k+1)} \leftarrow \mX_i^{(r, k)} + \eta \text{lmo} \left( \mM_i^{(r, k+1)} - \mC_i^{(r)} + \mC^{(r)} \right)$.
        \EndFor
        \State $\mC_i^{(r+1)} \leftarrow \mM_i^{(r, K)}$
        \State send $\mX_i^{(r, K)}$ and $\mC_i^{(r+1)}$ to the central server.
    \EndFor  \emph { \color{gray} (end clients, back to the server)}
    \For{$i \in [n] \setminus \mathcal{S}_r$}
        \State $\mC_i^{(r+1)} \leftarrow \mC_i^{(r)}$ and $\mM_i^{(r, K)} \leftarrow \mM_i^{(r-1, K)}$.
    \EndFor
    \State $\mC^{(r+1)} \leftarrow \mC^{(r)} + \frac{1}{N} \sum_{i \in \mathcal{S}_r} \left( \mC_i^{(r+1)} - \mC_i^{(r)} \right)$.
    \State $\mX^{(r+1)} \leftarrow \frac{n-S}{n} \mX^{(r)} + \frac{1}{n} \sum_{i \in \mathcal{S}_r} \mX_i^{(r, K)}$.
\EndFor
\end{algorithmic}
\end{algorithm}

\section{FedMuon}
\label{sec:fedmuon}

In the previous section, we showed that due to the bias caused by the LMO, \textsc{LocalMuon} does always converge.
In this section, in Algorithm~\ref{algorithm:proposed_method} we propose \proposed{}, which mitigates this issue and provably converges to the stationary point.

Instead of applying the LMO to the momentum alone, we apply the LMO to the bias corrected version of the momentum (line 8) in \cref{algorithm:proposed_method}.
Similarly to SCAFFOLD \citep{karimireddy2020scaffold} we introduce control variates $\mC_i^{(r)}$ and $\mC^{(r)}$ to estimate the directions of the local client gradients $\nabla f_i (\mX^{(r)})$ and the global gradient $\nabla f (\mX^{(r)})$, respectively. Given that the local momentum parameters $\mM_i^{(r, k+1)}$ estimate local gradients $\nabla f_i (\mX^{(r, k)})$,
the corrected update, $\mM_i^{(r, k+1)} - \mC_i^{(r)} + \mC^{(r)}$ is a good estimation of the full gradient $\nabla f (\mX^{(r, k)})$, mitigating the issue of local bias.
When we remove the LMO and set $\alpha = 1$, \proposed{} recovers vanilla SCAFFOLD \citep{karimireddy2020scaffold}.
It is important to note that there are several papers that apply the momentum to SCAFFOLD \citep{cheng2024momentum,karimireddy2021breaking}, however all of them incorporate momentum at the central server, differing from our proposed \proposed{}.

\section{Convergence Analysis}

\subsection{Assumptions}
We first summarize the assumptions that we use in our theoretical results.
As is common in the prior literature analyzing optimizers with LMO \citep{pethick2025training,riabinin2025gluonmakingmuon}, we use the following smoothness assumption.
Note that the norm here is the same as the one used in the LMO.

\begin{assumption}
\label{assumption:smoothness}
There exists $L \geq 0$ so that it holds for any $\mX, \mY \in \mathcal{X}$
\begin{align*}
    \| \nabla f_i (\mX) - \nabla f_i (\mY) \|_\star \leq L \| \mX - \mY \|.
\end{align*}
\end{assumption}

Since we consider non-Euclidean norms, we measure gradient differences in the dual norm, while parameter differences are measured in the primal norm \citep[cf.][]{nesterov2018lectures,xie2024implicit}.
However, as shown in \cref{remark:relationship_between_two_norm}, any two norms are equivalent in finite dimensions, and thus the class of functions satisfying Assumption \ref{assumption:smoothness} and the conventional smoothness assumptions (formulated for Euclidean norms) is the same (see \cref{remark:smoothness}).

\begin{remark}[{\cite[Theorem 3.1]{conway2019course}}]
\label{remark:relationship_between_two_norm}
If $\mathcal{X}$ is a finite-dimensional vector space over $\mathbb{F}$, then for any two norms $\| \cdot \|_p$ and $\| \cdot \|_q$, there exist $c, C \geq 0$ such that $c \| \mX \|_p \leq \| \mX \|_q \leq C \| \mX \|_p$ for all $\mX \in \mathcal{X}$. 
\end{remark}

\begin{remark}
\label{remark:smoothness}
If it holds that $\| \nabla f_i (\mX) - \nabla f_i (\mY) \| \leq C L \| \mX - \mY \|$ for any $\mX, \mY \in \mathcal{X}$, then $f_i$ satisfies \cref{assumption:smoothness} where $C \coloneqq \sup_{\mX \in \mathcal{X}} \frac{\| \mX \|_\star}{\| \mX \|}$.
If $f_i$ satisfies \cref{assumption:smoothness}, it holds that $\| \nabla f_i (\mX) - \nabla f_i (\mY) \| \leq \frac{L}{c} \| \mX - \mY \|$ for any $\mX, \mY \in \mathcal{X}$ where $c \coloneqq \sup_{\mX \in \mathcal{X}} \tfrac{\| \mX \|}{\| \mX \|_\star}$.
\end{remark}

For the analysis of \proposed{}, we often use the trace norm and Frobenius norm.
The following inequality holds between the Frobenius norm and the trace norm.
\begin{example}
For any $\mX \in \mathbb{R}^{d_1 \times d_2}$, it holds that $\| \mX \|_F \leq \| \mX \|_\text{trace} \leq \sqrt{\min \{ d_1, d_2 \}}\| \mX \|_F$.
\end{example}

For the stochastic gradient noise, we use the following assumption, which is quite common in the optimization literature, e.g., \citep{bubeck2015convex}.

\begin{assumption}
The stochastic gradient is unbiased, i.e., $\mathbb{E} [\nabla F_i (\mX; \xi_i) ] = \nabla f_i (\mX)$ for any $\mX \in \mathcal{X}$. Then, there exists $\sigma \geq 0$ so that it holds for any $\mX \in \mathcal{X}$
\label{assumption:stochastic_noise}
\begin{align*}
    \mathbb{E}_{\xi_i \sim \mathcal{D}_i} \left\| \nabla F_i (\mX ; \xi_i) - \nabla f_i (\mX) \right\|^2_F \leq \sigma^2. 
\end{align*}
\end{assumption}

\subsection{Convergence Result}

We provide the convergence rate of \proposed{} in \Cref{theorem:convergence_analysis}. 
For simplicity, we present the results for the special case $S = n$, where all clients participate during the training. 
The general case with arbitrary $S$ is provided in \cref{lemma:convergence_rate} in \cref{sec:proof_of_main_theorem}. 
The proof is deferred to \cref{sec:proof_of_main_theorem}.

\begin{theorem}
\label{theorem:convergence_analysis}
Consider \cref{algorithm:proposed_method}.
We define $\mX^{(r, k)} = \frac{1}{n} \sum_{i=1}^n \mX_i^{(r, k)}$.
Note that $\mX^{(r+1)} = \mX^{(r, K)}$.
Suppose that $n=S$ and \cref{assumption:smoothness,assumption:stochastic_noise} hold,  $\mC_i^{(0)} \coloneqq \mM_i^{(0, 0)}$ and $\mC^{(0)} \coloneqq \tfrac{1}{n} \sum_{i=1}^n \mC_i^{(0)}$, there exists $\eta$ and $\alpha$ so that it satisfies
\begin{align*}
    \frac{1}{R K} \sum_{r=0}^{R-1} \sum_{k=0}^{K-1} \mathbb{E} \left\| \nabla f (\mX^{(r, k)}) \right\|_\star 
    &\leq \mathcal{O} \left( 
        \left( \frac{L r_0 \tilde{\sigma}^2}{n R K} \right)^\frac{1}{4}
        + \left(  \frac{L r_0 \tilde{\sigma}}{R \sqrt{K}} \right)^\frac{1}{3}
        + \left( \frac{L r_0}{R} \right)^\frac{1}{2} \right. \\
        &\quad \left. + \tilde{\sigma}_0 \left[ \frac{1}{R} 
        + \left( \frac{\tilde{\sigma}^2 K}{L r_0 R n} \right)^{\frac{1}{2}}
        + \left( \frac{\tilde{\sigma}^2 K^2}{L r_0 R^2}\right)^\frac{1}{3} \right]
    \right),
\end{align*}
where $r_0 \coloneqq f (\mX^{(0)}) - f^\star, \rho \coloneqq \sup_{\mX \in \mathcal{X}} \tfrac{\| \mX \|_\star}{\| \mX \|_F}, \tilde{\sigma} \coloneqq \rho \sigma$, $\tilde{\sigma}_0 \coloneqq \rho \sigma_0$, and $\sigma_0^2 \coloneqq \tfrac{1}{n} \sum_{i=1}^n \mathbb{E} \| \mM_i^{(0,0)} - \nabla f_i (\mX^{(0)}) \|_F^2$.
\end{theorem}

\paragraph{Discussion:}
Unlike \textsc{LocalMuon}, \cref{theorem:convergence_analysis} shows that \proposed{} can mitigate the issue that the LMO is a biased operator and can converge to the stationary point.
The dominant term is $\mathcal{O} (\tfrac{L r_0 \tilde{\sigma}^2}{n R K})^\frac{1}{4}$, which is almost the same as the terms appearing in the rate of FedAvg and SCAFFOLD \citep{karimireddy2020scaffold}, and the convergence rate is improved as the number of clients $n$ increases.
The only difference is that the convergence rate of \proposed{} depends on $\rho$, while this is because \cref{theorem:convergence_analysis} analyzes the dual norm of the gradient.
For instance, when the norm is the Frobenius norm, the dual norm is also the Frobenius norm and $\rho = 1$.
The last three terms arise from the initial error $\sigma_0$, which diminish faster than the other terms as the number of rounds $R$ increases.

We consider the case where the parameter space is $\mathcal{X} = \mathbb{R}^{d_1 \times d_2}$ and the spectral norm is used, as in Muon \citep{liu2025muon}. Since the dual of the spectral norm is the trace norm, we have $\| \nabla f (\mX) \|_F \leq \| \nabla f(\mX) \|_\text{trace}$. Consequently, \proposed{} can converge faster than SCAFFOLD in certain cases.
For instance, if the stochastic noise $\sigma$ is sufficiently small, \proposed{} converges as:
\begin{align*}
    \frac{1}{R K} \sum_{r=0}^{R-1} \sum_{k=0}^{K-1} \mathbb{E} \left\| \nabla f (\mX^{(r, k)}) \right\|_\text{trace} 
    &\leq \mathcal{O} \left( \left( \frac{r_0 L}{R}  \right)^\frac{1}{2} \right),
\end{align*}
and SCAFFOLD converges as follows (see Theorem 3 in \citep{karimireddy2020scaffold}):
\begin{align*}
    \frac{1}{R K} \sum_{r=0}^{R-1} \sum_{k=0}^{K-1} \mathbb{E} \left\| \nabla f (\mX^{(r, k)}) \right\|_F 
    &\leq \mathcal{O} \left( \left( \frac{r_0 L_F}{R}  \right)^\frac{1}{2} \right).
\end{align*}
where $L_F$ refers to the smoothness of $f_i$ with respect to the Frobenius norm.
Thus, when $L = \sup_{i \in [n], \mX,   \|\mU\|_{\text{sp}}\le 1}\langle \mU, \nabla^2 f_i(\mX) \mU \rangle \approx L_F$, i.e., when the Hessians have a few dominant singular values—equivalently, when they are approximately low-rank, then \proposed{} can converge faster than SCAFFOLD.
More precisely, the terms on the right-hand side are the same, and the only difference is the choice of the norm.
We stress that \cref{theorem:convergence_analysis} does not claim \proposed{} always converges faster, but it does suggest that in certain cases \proposed{} can outperform. This helps explain the strong empirical performance of Muon and \proposed{}.

\begin{algorithm}[b!]
\renewcommand\algorithmicindent{1.2em}
\caption{Newton-Schulz iteration}
\label{algorithm:ns_iteration}
\begin{algorithmic}[1]
\State \textbf{Input:} matrix $\mG$ and hyperparameters $a, b, c \in \mathbb{R}$.
\State $\mG^{(0)} \leftarrow \frac{\mG}{\| \mG \|_F}$. 
\For{$t \in \{ 0, 1, \cdots, T-1 \}$}
    \State $\mG^{(t+1)} \leftarrow a \mG^{(t)} + b (\mG^{(t)} {\mG^{(t)}}^\top) \mG^{(t)} + c (\mG^{(t)} {\mG^{(t)}}^\top)^2 \mG^{(t)}$. 
\EndFor
\State \textbf{Retern} $- \mG^{(T)}$
\end{algorithmic}
\end{algorithm}

\section{FedMuon with Inexact LMO}
\label{sec:fedmuon_with_inexact_lmo}

In the previous section, we considered the general case with an arbitrary norm and exact LMO. Here, we focus on the spectral norm, as in Muon \citep{liu2025muon}, and analyze \proposed{} when the LMO is only approximately solved via the Newton–Schulz iteration.
Then, thanks to the special property of spectral norm and Newton-Schulz iteration, we reveal that \proposed{} can converge to the stationary point regardless of how accurately we solve the LMO.

With the spectral norm, the LMO takes the following form:
\begin{align*}
    \text{lmo}_\text{muon} (\mX) \coloneqq \argmin_{\mY \in \{ \mY \in \mathbb{R}^{d_1 \times d_2} \mid \| \mY \|_\text{sp} \leq 1 \}} \langle \mX, \mY \rangle,
\end{align*}
Let the singular value decomposition of $\mX$ be $\mU \Sigma \mV$.
Then the LMO output is $- \mU \mV$, but computing this exactly is computationally expensive.
To address this, \citet{liu2025muon} proposed approximating the LMO via a fixed number of Newton–Schulz iterations (e.g., 5). 
The update rule of the Newton-Schulz iteration is described in \cref{algorithm:ns_iteration}.
Since the procedure involves only matrix multiplications, it can be efficiently executed on a GPU.
In the following, we analyze the convergence of \proposed{} when the LMO is solved approximately using Newton–Schulz iterations and characterize how inexactness impacts convergence.

Under the same assumption as in \cref{theorem:convergence_analysis}, we provide the convergence rate when we run Newton-Schulz iteration $T$ times to solve the LMO approximately and show how $T$ affects convergence.
For simplicity, we set $n=S$ and we use $a=\frac{15}{8}, b=-\frac{5}{4}$ and $c=\frac{3}{8}$ for the Newton-Schulz iteration, following the hyperparameter setting mentioned in \citet{amsel2025polar}.
For the general case of arbitrary $S$ we refer to \cref{lemma:convergence_rate_with_inexact_lmo} in \cref{sec:proof_of_theorem_with_inexact_lmo}. 

\begin{theorem}
\label{theorem:convergence_analysis_with_ns_iteration}
Consider \cref{algorithm:proposed_method} with the spectral norm and suppose that the LMO in line 8 is solved approximately using \cref{algorithm:ns_iteration} with $a=\frac{15}{8}, b=-\frac{5}{4}$, and $c=\frac{3}{8}$.
We define $\mX^{(r, k)} = \frac{1}{n} \sum_{i=1}^n \mX_i^{(r, k)}$.
Note that $\mX^{(r+1)} = \mX^{(r, K)}$.
Suppose that $n=S$ and \cref{assumption:smoothness,assumption:stochastic_noise} hold,  $\mC_i^{(0)} \coloneqq \mM_i^{(0, 0)}$ and $\mC^{(0)} \coloneqq \tfrac{1}{n} \sum_{i=1}^n \mC_i^{(0)}$.
Then, for any number of Newton-Schulz iteration $T \geq 0$, there exists $\eta$ and $\alpha$ so that it satisfies
\begin{align*}
    \frac{1}{R K} \sum_{r=0}^{R-1} \sum_{k=0}^{K-1} \mathbb{E} \left\| \nabla f (\mX^{(r, k)}) \right\|_p 
    &\leq \mathcal{O} \left( 
        \left( \frac{L r_0 \tilde{\sigma}^2}{n R K} \right)^\frac{1}{4}
        + \left( \frac{L r_0 \tilde{\sigma}}{R \sqrt{K}} \right)^\frac{1}{3}
        + \left( \frac{L r_0}{R} \right)^\frac{1}{2} \right. \\
        &\quad \left. + \tilde{\sigma}_0 \left[ \frac{1}{R} 
        + \left( \frac{\tilde{\sigma}^2 K}{L r_0 R n} \right)^{\frac{1}{2}}
        + \left( \frac{\tilde{\sigma}^2 K^2}{L r_0 R^2}\right)^\frac{1}{3} \right]
    \right),
\end{align*}
where $r_0 \coloneqq f (\mX^{(0)}) - f^\star, \rho \coloneqq \sqrt{\min \{ d_1, d_2 \}}, \tilde{\sigma} \coloneqq \rho \sigma$, $\tilde{\sigma}_0 \coloneqq \rho \sigma_0$, and $\sigma_0^2 \coloneqq \tfrac{1}{n} \sum_{i=1}^n \mathbb{E} \| \mM_i^{(0,0)} - \nabla f_i (\mX^{(0)}) \|_F^2$.
Then, $p$ is defined as follows:
\begin{align*}
    p &\coloneqq 1 + \frac{\log \left( 1 - (1 - \kappa)^{{1.5}^T} \right)}{\log \kappa}, \\
    \kappa &\coloneqq \min_{j,i,r,k} \frac{s_{j,i,r,k}}{\sqrt{\sum_{j'} s_{j',i,r,k}^2}} \; (> 0),
\end{align*}
where $\{ s_{j,i,r,k} \}_j$ are non-zero singular values of $\mM_i^{(r, k+1)} - \mC_i^{(r)} + \mC^{(r)}$. 
\end{theorem}

\begin{remark}
When $T=0$, $p=2$. 
As $T$ increase, $p$ monotonically decreases to $1$ for any $\kappa>0$.
\end{remark}

\begin{remark}
Recall that $\| \cdot \|_p$ is the Schatten $p$-norm.
For any $1 \leq p \leq q$, we have $\| \mA \|_q \leq \| \mA \|_p$.
Then, $\| \mA \|_p$ becomes $\| \mA \|_\text{trace}$ and $\| \mA \|_F$ when $p=1$ and $p=2$, respectively.
\end{remark}

\paragraph{Discussion:}
Surprisingly, the above theorem shows that \proposed{} converges to the stationary point, regardless of how many times we run the Newton-Schulz iteration.
The only difference between the case when we solve the LMO exactly, i.e., \cref{theorem:convergence_analysis}, and the case when we solve the LMO approximately, i.e., \cref{theorem:convergence_analysis_with_ns_iteration}, is that \cref{theorem:convergence_analysis} establishes the convergence in the trace norm of the gradient $\| \nabla f (\mX) \|_\text{trace} (= \| \nabla f (\mX) \|_1)$,\footnote{When $\| \cdot \|$ is the spectral norm, its dual norm is the trace norm.} while \cref{theorem:convergence_analysis_with_ns_iteration} establishes the convergence in the Schatten $p$-norm $\| \nabla f (\mX) \|_p$. We recover the convergence rate of Theorem~\ref{theorem:convergence_analysis} by setting $T \to \infty$, and therefore have $p \to 1$.
Since we have $\| \mA \|_q \leq \| \mA \|_p$ when $1 \leq p \leq q$, \cref{theorem:convergence_analysis_with_ns_iteration} implies that \proposed{} can converge faster when we increase the number of Newton-Schulz iterations $T$.
More specifically, since it holds that $\| \mA \|_1 \leq \sqrt{\min\{ d_1, d_2 \}} \| \mA \|_2$, solving the LMO accurately can improve the convergence rate by up to a factor of $\sqrt{\min\{ d_1, d_2 \}}$.
In our experiments, we will demonstrate that \proposed{} can train neural networks even if $T=0$, while \proposed{} can achieve higher accuracy as $T$ increases (see \cref{sec:effect_of_inexact_lmo}).
These observations are consistent with \cref{theorem:convergence_analysis_with_ns_iteration}.

The quantity of $(1 - \kappa)^{{1.5}^T}$ in the definition of $p$ measures how fast the Newton-Schulz iteration converges.
If we consider the worst case, $\kappa$ could be arbitrarily close to zero, and thus a large $T$ would be required to sufficiently decrease $p$.
However, the main implication of \cref{theorem:convergence_analysis_with_ns_iteration} is that increasing $T$ leads to an improved convergence rate.
Indeed, our experiments show that even increasing $T$ from $0$ to $1$ dramatically improves accuracy (see \cref{sec:effect_of_inexact_lmo}).

\paragraph{Comparison with Existing Analysis with Inexact LMO:}
There are many papers that analyzed the convergence rate of Muon, while most of them assumed that the LMO is exactly solved \citep{pethick2025training,riabinin2025gluonmakingmuon,shen2025convergence}.
The only study analyzing the rate with an inexact LMO is \citet{refael2025sumo}. However, they also assumed that we run Newton-Schulz iterations a certain number of times (see Lemma 3.3 and Remark 3.6 in \citep{refael2025sumo}).
Compared with these prior analyses, our novel analysis provides a stronger claim that \proposed{} can converge to the stationary point for any number of the Newton-Schulz iterations $T \geq 0$.
Furthermore, it is first observed by \cref{theorem:convergence_analysis_with_ns_iteration} that the different norms of the gradient are bounded depending on $T$.

\paragraph{Proof Sketch:}
In the following, we provide an intuition for why \proposed{} can converge for any $T \geq 0$.
If we solve the LMO exactly, we have
\begin{align}
    \label{eq:exact_lmo}
    \langle \mG, \text{lmo}_\text{muon} (\mG) \rangle = - \| \mG \|_\text{trace}, \qquad \left\| \text{lmo}_\text{muon} (\mG) \right\|_\text{sp} \leq 1.
\end{align}
The first equality holds from the definition of the dual norm (see \cref{lemma:lmo}), and the second inequality holds since the solution of the LMO satisfies the constraint.
Then, the output of the Newton-Schulz iteration satisfies the following (see \cref{lemma:inexact_lmo2}):
\begin{align}
    \label{eq:inexact_lmo]}
    - \| \mG \|_\text{trace} \leq \langle \mG, - \mG^{(T)} \rangle \leq - \| \mG \|_p, \qquad \left\| - \mG^{(T)} \right\|_\text{sp} \leq 1.
\end{align}
The above inequality indicates that even if we run the Newton-Schulz iteration only a few times to solve the LMO approximately, the output of the Newton-Schulz iteration is a proper direction to minimize the loss function, and \proposed{} can converge to the stationary point.
For instance, when $T=0$, the output of Newton-Schulz iteration is $- \tfrac{\mG}{\| \mG \|_F}$, which corresponds to the normalized gradient, and it is natural that \proposed{} can converge to the stationary point when $T=0$.
Then, if we run the Newton-Schulz iteration $T$ times, the output of the Newton-Schulz iteration comes close to the exact solution of LMO and remains a proper direction to minimize the loss function.
Thanks to this property, \proposed{} can converge to the stationary point for any number of Newton-Schulz iterations $T$.

\section{Experiment}
\begin{figure}[t]
    \vskip - 0.1 in
    \begin{minipage}[t]{\hsize}
        \centering
        \includegraphics[width=\linewidth]{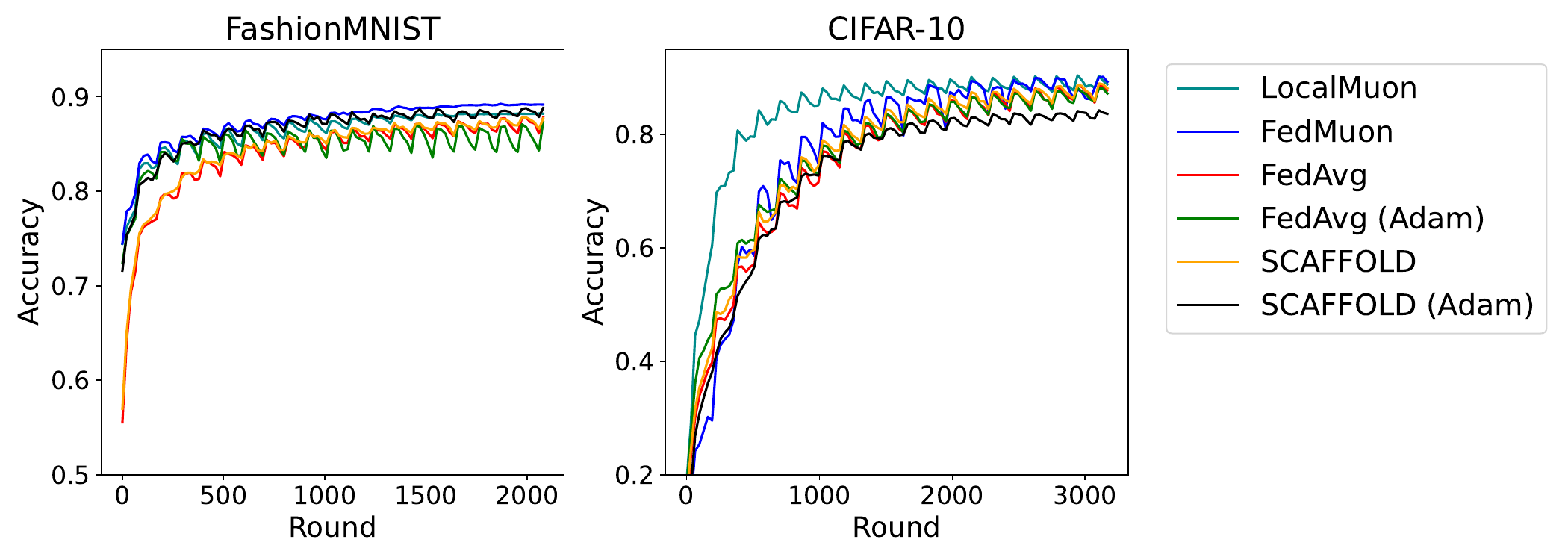}
        \vskip - 0.1 in
        \subcaption{Homogeneous Case ($\beta=10.0$)}
        \label{fig:iid}
    \end{minipage}
    \begin{minipage}[t]{\hsize}
        \centering
        \includegraphics[width=\linewidth]{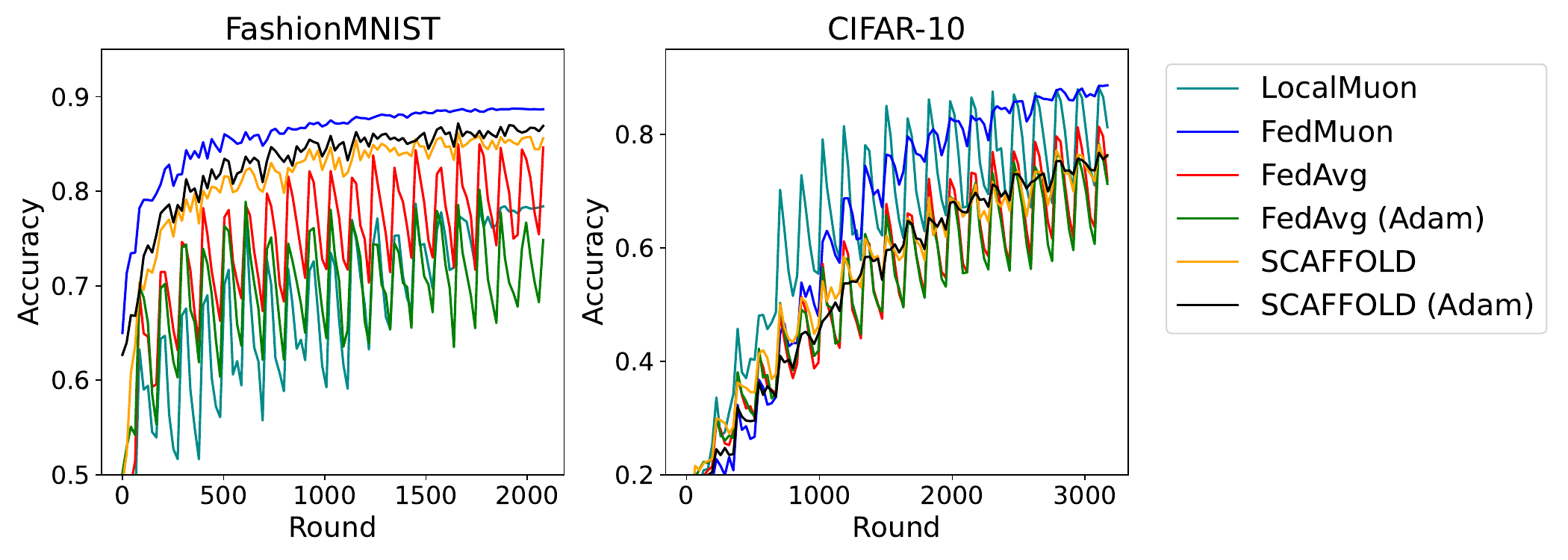}
        \vskip - 0.1 in
        \subcaption{Heterogeneous Case ($\beta=0.1$)}
        \label{fig:non_iid}
    \end{minipage}
    \caption{Training curves of various methods. For all settings, \proposed{} can achieve higher test accuracy than other methods.}
    \label{fig:main_results}
    \vskip - 0.1 in
\end{figure}
\subsection{Federated Learning Tasks}
\paragraph{Setup:}
We used FashionMNIST \citep{xiao2017fashion} and CIFAR-10 \citep{krizhevsky09learning} as training datasets, and used LeNet \citep{lecun1998gradient} for Fashion MNIST and ResNet-18 \citep{He2015} for CIFAR-10.
Following the prior paper \citep{hsieh2020non}, we used Group Normalization \citep{wu2018group} instead of Batch Normalization \citep{ioffe2015batch} for ResNet-18.
We set the number of clients $n$ to $16$ and sampled $S=8$ clients every round.
We set the number of local steps $K$ to $5$ and set the number of epochs to $100$ and $200$ for FashionMNIST and CIFAR-10, respectively.
Following the prior paper \citep{hsu2019measuring}, we distributed the training dataset to clients by using Dirichlet distributions with hyperparameter $\beta$.
As $\beta$ approaches zero, each clients come to have a different training dataset.
We tuned the stepsize by grid search. See \cref{sec:hyperparameter_tuning} for details.
The experiments were repeated with two different seed values, and we reported the average.

\paragraph{Comparison Methods:}
We compared the following methods:
(1) FedAvg \citep{mcmahan2017communication}: We used Momentum SGD as the optimizer.
(2) FedAvg (Adam): We used Adam as the optimizer of FedAvg.
(3) SCAFFOLD \citep{karimireddy2020scaffold}: We used Momentum SGD as the optimizer.
(4) SCAFFOLD (Adam): We used Adam as the optimizer of SCAFFOLD.
(5) \proposed{}: Our proposed method.
Following the suggestion of \citet{liu2025muon}, we changed the scale of the stepsize per layer, depending on the dimension.

\paragraph{Results:}
We show the results in \cref{fig:main_results}.
The results indicate that \proposed{} can perform the best for all settings.
By comparing \proposed{} with FedAvg (Adam) and SCAFFOLD (Adam), \proposed{} achieved the highest accuracy, which can demonstrate that Muon is also beneficial in the federated learning setting.
By comparing \proposed{} and \textsc{LocalMuon}, \textsc{LocalMuon} performed well in the homogeneous setting, but did not match the performance of \proposed{} in the heterogeneous setting.
This observation is consistent with the discussion in \cref{sec:local_muon}, where we show that \textsc{LocalMuon} does not converge to the stationary point in the heterogeneous setting.
These observations were consistent with \cref{theorem:lower_bound}.

\subsection{Effect of Inexact LMO}
\label{sec:effect_of_inexact_lmo}

Next, we evaluate how the number of Newton-Schulz iterations $T$ affects the performance.
\Cref{fig:inexact_lmo} shows the training curves of \proposed{} with different $T$.
In the homogeneous setting, the highest accuracy was achieved when $T=4$, and in the heterogeneous setting, the highest accuracy was achieved when $T=2$.
Thus, we can observe that solving the LMO accurately can improve the performance.
Notably, \proposed{} already worked with $T=0$, but increasing $T$ from $0$ to $1$ led to a significant improvement in accuracy.
These observations were consistent with \cref{theorem:convergence_analysis_with_ns_iteration}, which shows that \proposed{} can converge for any $T$ and converge faster as $T$ increases.

\begin{figure}[t!]
    \centering
    \vskip - 0.1 in
    \includegraphics[width=\linewidth]{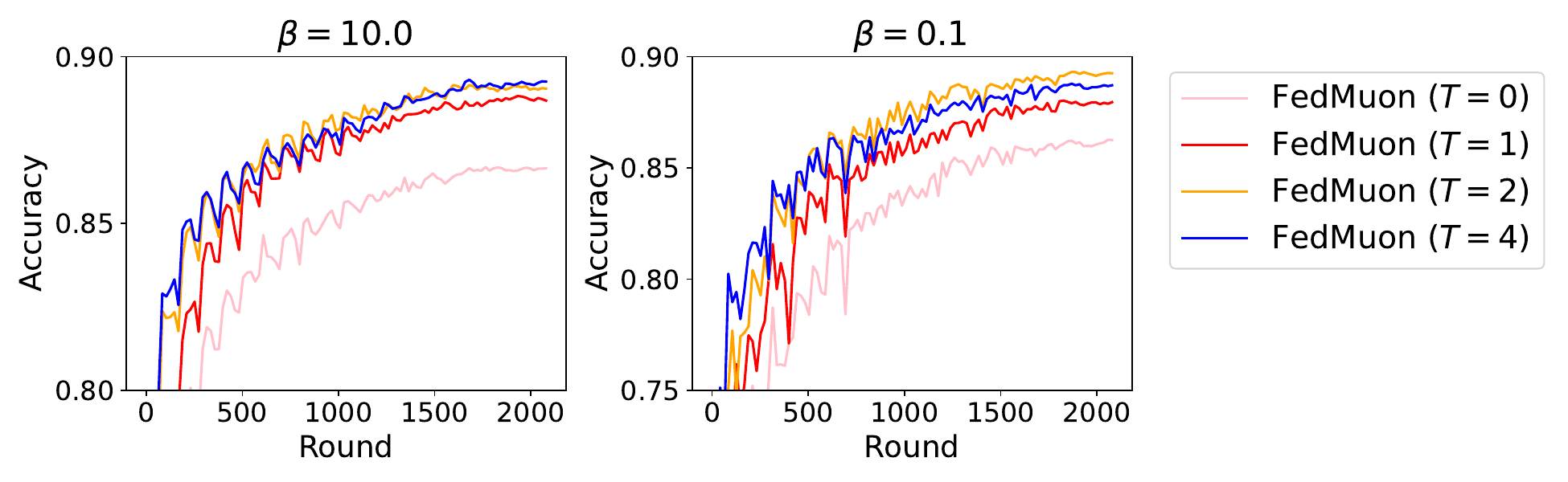}
    \vskip -0.1 in
    \caption{Training curves of \proposed{} with various number of Newton-Schulz iterations. We used FashionMNIST and LeNet.}
    \label{fig:inexact_lmo}
    \vskip -0.2 in
\end{figure}

\section{Conclusion}
In this paper, we study the federated learning methods with the LMO and propose \proposed{}.
We first propose directly plugging the optimization methods with the LMO into FedAvg, which we referred to as \textsc{LocalMuon},
and show that \textsc{LocalMuon} cannot converge to the stationary point since the LMO is a biased operator.
We then propose \proposed{} to solve this issue and show that \proposed{} can converge to the stationary point.
We analyze the convergence rate of \proposed{} and reveal how the approximate solution of the LMO affects the convergence behavior.
Notably, we show that \proposed{} can converge for any number of Newton-Schulz iterations, and \proposed{} can converge faster as we solve the LMO more accurately.
Throughout the experiments, we demonstrated the effectiveness of \proposed{} and verified our theoretical discovery.

\bibliography{iclr2025_conference}
\bibliographystyle{iclr2025_conference}

\newpage

\appendix

\section{LLM Usage}
We used LLM for proofreading, and it did not contribute to the content of the paper itself.

\section{Related Work}
\label{sec:related_work}

\paragraph{Federated Learning:}
The simplest algorithm for federated learning is FedAvg \citep{mcmahan2017communication,stich2019local}.
The main challenge of federated learning is reducing communication between the central server and clients.
Various techniques such as client sampling \citep{gu2021fast,chen2022optimal,zhang2023heterogeneity}, multiple local steps \citep{woodworth2020is,koloskova2020unified,jian2024stabilized,jiang2024federated}, and communication compression \citep{alistarh2017qagd,stich2018sparsified,karimireddy2019error,vogels2019powersgd,he2023unbiased,gao2024econtrol} have been studied to reduce the communication costs.
However, FedAvg still requires a huge amount of communication when clients have different training datasets.
Many papers proposed federated learning methods that are robust to data heterogeneity \citep{karimireddy2020scaffold,jian2024stabilized,jiang2024federated}.
The seminal work is SCAFFOLD \citep{karimireddy2020scaffold}, which can converge regardless of data heterogeneity.
Besides these methods, asynchronous methods \citep{koloskova2022sharper,mishchenko2022asynchronous,islamov2024asgrad} and decentralized methods \citep{nedic2017achieving,tang2018d2,tang2018communication,koloskova2020unified,takezawa2023momentum} have been widely studied to further improve the efficiency.

\paragraph{Adaptive Optimization Methods:}
Using adaptive optimization methods is standard for training neural networks efficiently \citep{amari1998natural,rachel2020adagrad,duchi2011adaptive,kingma2017adam,loshchilov2018decoupled,zaheer2018adaptive,zhuang2020adablief,defazio2024the,rodomanov2024universality}.
Over the last decade, Adam \citep{kingma2017adam} and AdamW \citep{loshchilov2018decoupled} are the most widely used, but recently, Shampoo \citep{gupta2018shampoo} won the External Tuning Task of AlgoPerf \citep{dahl2025benchmarking} and is attracting considerable attention \citep{shi2023distributed,vyas2025soap,ishikawa2024on}.
Muon \citep{liu2025muon} can be regarded as the simplified version of Shampoo, and many papers have demonstrated that Muon can train neural networks faster than Adam, AdamW, and Shampoo \citep{liu2025muon,pethick2025training,amsel2025polar,liu2025cosmos,ma2025swan,amsel2025polar,grishina2025accelerating}.
Using Muon in distributed environments is one of the popular topics \citep{therien2025muloco,ahn2025dion}.
Specifically, \citet{ahn2025dion} proposed a method to solve the LMO in a distributed way, and \citet{therien2025muloco} proposed MuLoCo, which extends Muon by allowing clients to perform several steps before averaging the parameters as in \textsc{LocalMuon}.
However, since they consider settings where all clients have the same dataset, their objective differs from ours.
As we explained in \cref{sec:local_muon}, because the LMO is a biased operator, bias correction mechanisms used in \proposed{} are necessary in federated learning, in which clients have different training datasets.

\newpage

\section{Pseudo Code}
\label{sec:algorithm_of_local_muon}

\begin{algorithm}[h]
\renewcommand\algorithmicindent{1.2em}
\caption{FedAvg \citep{mcmahan2017communication}}
\label{algorithm:federated_averaging}
\begin{algorithmic}[1]
\State \textbf{Input:} the total number of clients $n$, the number of sampled clients $S$, and the number of local steps $K$.
\For{$t \in \{ 0, 1, \cdots, T \}$} \emph {\color{gray} (at the server)}
    \State sample $S$ clients $\mathcal{S}_{r} \subset [n]$. 
    \For{$i \in \mathcal{S}_r$} \emph{ \color{gray} (at the clients)}
        \State $\mX_i^{(r, 0)} \leftarrow \mX^{(r)}$.
        \For{$k = 0, 1, \cdots, K-1$}
            \State $\mX_i^{(r, k+1)} \leftarrow \mX_i^{(r, k)} - \eta \nabla F_i (\mX_i^{(r, k)} ; \xi_i^{(r, k)})$.
        \EndFor
    \EndFor \emph{ \color{gray} (end clients, back to the server)}
    \State $\mX^{(r+1)} \leftarrow \frac{n-S}{n} \mX^{(r)} + \frac{1}{n} \sum_{i \in \mathcal{S}_r} \mX_i^{(r, K)}$.
\EndFor
\end{algorithmic}
\end{algorithm}

\begin{algorithm}[h]
\renewcommand\algorithmicindent{1.2em}
\caption{LocalMuon}
\label{algorithm:localmuon}
\begin{algorithmic}[1]
\State \textbf{Input:} the total number of clients $n$, the number of sampled clients $S$, and the number of local steps $K$.
\For{$r \in \{ 0, 1, \cdots, R-1 \}$}  \emph {\color{gray} (at the server)}
    \State sample $S$ clients $\mathcal{S}_{r} \subset [n]$. 
    \For{$i \in \mathcal{S}_r$} \emph{ \color{gray} (at the clients)}
        \State $\mX_i^{(r, 0)} \leftarrow \mX^{(r)}$ and $\mM_i^{(r, 0)} \leftarrow \mM_i^{(r-1, K)}$
        \For{$k = 0, 1, \cdots, K-1$}
            \State $\mM_i^{(r, k+1)} \leftarrow (1 - \alpha) \mM_i^{(r, k)} + \alpha \nabla F_i (\mX_i^{(r, k)} ; \xi_i^{(r, k)})$.
            \State $\mX_i^{(r, k+1)} \leftarrow \mX_i^{(r, k)} + \eta \text{lmo} \left( \mM_i^{(r, k+1)} \right)$.
        \EndFor
        \State $\mC_i^{(r+1)} \leftarrow \mM_i^{(r, K)}$
    \EndFor 
    \For{$i \in [n] \setminus \mathcal{S}_r$}
        \State $\mM_i^{(r, K)} \leftarrow \mM_i^{(r-1, K)}$.
    \EndFor \emph{ \color{gray} (end clients, back to the server)}
    \State $\mX^{(r+1)} \leftarrow \frac{n-S}{n} \mX^{(r)} + \frac{1}{n} \sum_{i \in \mathcal{S}_r} \mX_i^{(r, K)}$.
\EndFor
\end{algorithmic}
\end{algorithm}

\newpage

\section{Proof of \cref{theorem:lower_bound}}
\label{sec:proof_of_lower_bound}

\begin{proof}
We consider the setting where $n = 2$, $d=1$, and the norm is the Euclidean norm.
In this case, we have
\begin{align*}
    \text{lmo} (x) = \frac{x}{|x|}. 
\end{align*}
Then, we consider the case where $f_1$ and $f_2$ are defined as follows:
\begin{align*}
    f_1 (x) &\coloneqq \frac{x^2}{2}, \\
    f_2 (x) &\coloneqq \frac{(x+a)^2}{2}.    
\end{align*}
When $\mM_i^{(0)} = 0$ and $\mX^{(0)} = - \frac{a}{4}$, we have
\begin{align*}
    \mM_1^{(1)} &= -\frac{\alpha a}{4}, \\
    \mM_2^{(1)} &= \frac{3\alpha a}{4}, \\
    \text{lmo} (\mM_1^{(1)}) + \text{lmo} (\mM_2^{(1)}) 
    &= 0,
\end{align*}
where we use $\alpha \in (0, 1]$.

Thus, the parameter does not change, i.e., $\mX^{(1)} = \mX^{(0)}$.
For the next round, we have 
\begin{align*}
    \mM_1^{(2)} &= - \frac{a}{4} \left( \alpha + \alpha (1 - \alpha) \right), \\
    \mM_2^{(2)} &= \frac{3 a}{4} \left( \alpha + \alpha (1 - \alpha) \right).
\end{align*}
Then, since it holds the following as in the first round:
\begin{align*}
    \text{lmo} (\mM_1^{(2)}) + \text{lmo} (\mM_2^{(2)}) 
    &= 0.
\end{align*}
The parameter does not change. 
Due to the above discussion, the parameter does not change for any $r$.
Now, we have
\begin{align}
    \| \nabla f (\mX^{(r)}) \|^2 = \frac{a^2}{16}.
\end{align}
Then, using $\frac{1}{2} \sum_{i=1}^2 \| \nabla f (\mX^\star) \|^2 = \frac{5 a^2}{16}$, we obtain the desired result.
\end{proof}

\newpage
\section{Proof of \cref{theorem:convergence_analysis}}
\label{sec:proof_of_main_theorem}

\subsection{Notation}
\label{sec:notation}
In this section, we use the following notation.

\begin{align}
    \mX^{(r, k)} &= \frac{n - S}{n} \mX^{(r)} + \frac{1}{n} \sum_{i \in \mathcal{S}_r} \mX_i^{(r, k)}, \\
    \mG_i^{(r, k+1)} &= \mM_i^{(r, k+1)} - \mC_i^{(r)} + \mC^{(r)}, \\
    \mD_i^{(r, k+1)} &= \text{lmo} \left( \mG_i^{(r, k+1)} \right).
\end{align}

\subsection{Useful Lemma}

\begin{lemma}
\label{lemma:lmo}
For any $\mX \in \mathcal{X}$, we have
\begin{align*}
    \left\langle \mX, \text{lmo} (\mX) \right\rangle =  - \| \mX \|_\star. 
\end{align*}
\end{lemma}
\begin{proof}
From the definition of $\text{lmo} (\cdot)$, we have
\begin{align*}
    \left\langle \mX, \text{lmo} (\mX) \right\rangle 
    &= \min_{\mY \in \{ \mY \in \mathcal{X} \mid \| \mY \| \leq 1 \}} \langle \mX, \mY \rangle \\
    &= - \max_{\mY \in \{ \mY \in \mathcal{X} \mid \| \mY \| \leq 1 \}} \langle - \mX, \mY \rangle \\
    &= - \| - \mX \|_\star \\
    &= - \| \mX \|_\star.
\end{align*}
\end{proof}

\begin{lemma}
\label{lemma:sum_of_geometric_series}
For any $k \geq 0, R \geq 0$ and $\alpha \in (0, 1]$, we have
\begin{align*}
    \sum_{r=0}^R k (1 - \alpha)^{k r} \leq \frac{1}{\alpha} + k.
\end{align*}
\end{lemma}
\begin{proof}
We have
\begin{align*}
    \sum_{r=0}^R k (1 - \alpha)^{k r} \leq \frac{k}{1 - (1 - \alpha)^k}.
\end{align*}
Then, using $(1 - \alpha)^k \leq e^{- \alpha k} \leq \tfrac{1}{1 + \alpha k}$, we obtain the desired result. 
\end{proof}

\begin{lemma}
\label{lemma:inner_product}
For any $\mA, \mB \in \mathcal{X}$, we have
\begin{align*}
    \langle \mA, \mB \rangle \leq \| \mA \| \| \mB \|_\star.
\end{align*}
\end{lemma}
\begin{proof}
We have
\begin{align*}
    \langle \mA, \mB \rangle = \| \mA \| \left\langle \frac{\mA}{\| \mA \|}, \mB \right\rangle
    \leq \| \mA \| \| \mB \|_\star.
\end{align*}
\end{proof}

\begin{lemma}
\label{lemma:smoothness}
Suppose that \cref{assumption:smoothness} holds. Then, it holds that for any $\mX, \mY \in \mathcal{X}$,
\begin{align}
    f_i (\mX) \leq f_i (\mY) + \langle \nabla f_i (\mY), \mX - \mY \rangle + \frac{L}{2} \left\| \mX - \mY \right\|.
\end{align}
\end{lemma}
\begin{proof}
Using the Fundamental Theorem of Calculus, we have
\begin{align*}
    f (\mX) &= f (\mY) + \int_{t=0}^1 \left\langle \nabla f (\mY + t (\mX - \mY)), \mX - \mY \right\rangle dt\\
    &= f (\mY) + \left\langle \nabla f (\mY), \mY - \mX \right\rangle  + \int_{t=0}^1 \left\langle \nabla f (\mY + t (\mX - \mY)) - \nabla f (\mY), \mX - \mY \right\rangle dt \\
    &\leq f (\mY) + \left\langle \nabla f (\mY), \mY - \mX \right\rangle  + \int_{t=0}^1 \left\| \nabla f (\mY + t (\mX - \mY)) - \nabla f (\mY) \|_\star \| \mX - \mY \right\| dt \\
    &\leq f (\mY) + \left\langle \nabla f (\mY), \mY - \mX \right\rangle  + \int_{t=0}^1 L t \left\| \mX - \mY \right\|^2 dt \\
    &= f (\mY) + \left\langle \nabla f (\mY), \mY - \mX \right\rangle  + \frac{L}{2} \left\| \mX - \mY \right\|^2,
\end{align*}
where we use \cref{lemma:inner_product,assumption:smoothness} for the first and second inequalities, respectively.
\end{proof}

\subsection{Main Proof}

\begin{lemma}
\label{lemma:norm_of_update}
Suppose that both $r = r'$ and $k \geq k'$ hold, or $r > r'$ holds. Then, we have
\begin{align*}
    \left\| \mX^{(r, k)} - \mX^{(r', k')} \right\| \leq \frac{\eta S}{n} ( (r - r')K + k - k')
\end{align*}
\end{lemma}
\begin{proof}
From the update rule of $\mX^{(r, k)}$ and $\mX^{(r', k')}$, we have
\begin{align*}
    \mX^{(r, k)} 
    &= \frac{n - S}{n} \mX^{(r)} + \frac{1}{n} \sum_{i \in \mathcal{S}_r} \mX_i^{(r, k)} \\
    &= \mX^{(r)} + \frac{\eta}{n} \sum_{i \in \mathcal{S}_r} \sum_{k''=1}^{k} \mD_i^{(r, k'')} \\
    &= \mX^{(r')} + \frac{\eta}{n} \sum_{i \in \mathcal{S}_r} \sum_{k''=1}^{k} \mD_i^{(r, k'')} + \frac{\eta}{n} \sum_{r''=r'}^{r-1} \sum_{i \in \mathcal{S}_{r''}} \sum_{k''=1}^{K} \mD_i^{(r'', k'')}, \\
    \mX^{(r', k')} 
    &= \mX^{(r')} + \frac{\eta}{n} \sum_{i \in \mathcal{S}_{r'}} \sum_{k''=1}^{k'} \mD_i^{(r', k'')}.
\end{align*}
Thus, we have
\begin{align*}
    \left\| \mX^{(r, k)} - \mX^{(r', k')} \right\|
    &= \left\| \frac{\eta}{n} \sum_{i \in \mathcal{S}_r} \sum_{k''=1}^{k} \mD_i^{(r, k'')} + \frac{\eta}{n} \sum_{r''=r'+1}^{r-1} \sum_{i \in \mathcal{S}_{r''}} \sum_{k''=1}^{K} \mD_i^{(r'', k'')}
    + \frac{\eta}{n} \sum_{i \in \mathcal{S}_{r'}} \sum_{k''=k'+1}^{K} \mD_i^{(r', k'')} \right\| \\
    &\leq \frac{\eta S}{n} ((r - r')K + k - k'),
\end{align*}
where we use $\| \mD_i^{(r, k)} \| = 1$ for any $r$ and $k$.
\end{proof}

\begin{lemma}
\label{lemma:norm_of_update_of_client}
Suppose that both $r = r'$ and $k \geq k'$ hold, or $r > r'$ holds. Then, we have
\begin{align*}
    \left\| \mX_i^{(r, k)} - \mX_i^{(r', k')} \right\| \leq (r - r' + 2) K \eta.
\end{align*}
\end{lemma}
\begin{proof}
We have
\begin{align*}
    \left\| \mX_i^{(r, k)} - \mX_i^{(r', k')} \right\| 
    &\leq \left\| \mX_i^{(r, k)} - \mX_i^{(r, 0)} \right\| 
    + \left\| \mX_i^{(r, 0)} - \mX_i^{(r', 0)} \right\| 
    + \left\| \mX_i^{(r', k')} - \mX_i^{(r', 0)} \right\| \\
    &= \left\| \mX_i^{(r, k)} - \mX_i^{(r, 0)} \right\| 
    + \left\| \mX^{(r, 0)} - \mX^{(r', 0)} \right\| 
    + \left\| \mX_i^{(r', k')} - \mX_i^{(r', 0)} \right\| \\
    &\leq \eta ( k + k')  
    + \left\| \mX^{(r, 0)} - \mX^{(r', 0)} \right\|.
\end{align*}
Using \cref{lemma:norm_of_update}, we obtain the desired result.
\end{proof}

\begin{lemma}
\label{lemma:descent_lemma}
Suppose that \cref{assumption:smoothness,assumption:stochastic_noise} hold.
Then, when $r \geq 1$, we have
\begin{align*}
    \mathbb{E} f (\mX^{(r, k+1)})
    &\leq \mathbb{E} f (\mX^{(r, k)}) - \frac{\eta S}{n} \left\| \nabla f (\mX^{(r, k)}) \right\|_\star 
    + 2 L K \left( \frac{S}{n} \right)^2 \eta^2 \\
    &\quad + 2 \left( \frac{S}{n} \right) \eta \mathbb{E} \left\| \nabla f (\mX^{(r-1, K-1)}) - \mC^{(r)} \right\|_\star 
    + \frac{2 \eta}{n} \mathbb{E} \sum_{i \in \mathcal{S}_r} \left\| \mM_i^{(r, k+1)} - \mC_i^{(r)} \right\|_\star
    + \frac{L}{2} \left( \frac{S}{n} \right) \eta^2.
\end{align*}
When $r=0$, we have
\begin{align*}
    \mathbb{E} f (\mX^{(0, k+1)})
    &\leq \mathbb{E} f (\mX^{(0, k)}) - \frac{\eta S}{n} \left\| \nabla f (\mX^{(0, k)}) \right\|_\star 
    + 2 L K \left( \frac{S}{n} \right)^2 \eta^2  \\
    &\quad + \frac{2 \eta}{n} \mathbb{E} \sum_{i \in \mathcal{S}_0} \left\| \mM_i^{(0, k+1)} - \mC_i^{(0)} \right\|_\star
    + \frac{L}{2} \left( \frac{S}{n} \right) \eta^2 
    + 2 \left( \frac{S}{n} \right) \rho \sigma_0 \eta.
\end{align*}
\end{lemma}
\begin{proof}
We have
\begin{align*}
    &\mathbb{E}_{r, k} f (\mX^{(r, k+1)}) \\
    &= \mathbb{E}_{r, k} f \left( \mX^{(r, k)} + \frac{\eta}{n} \sum_{i \in \mathcal{S}_r} \mD_i^{(r, k)} \right) \\
    &\leq f (\mX^{(r, k)}) + \frac{\eta}{n} \mathbb{E}_{r, k} \sum_{i \in \mathcal{S}_r} \left\langle \nabla f (\mX^{(r, k)}), \mD_i^{(r, k)} \right\rangle + \frac{L \eta^2}{2 n} \mathbb{E}_{r, k} \sum_{i \in \mathcal{S}_r} \left\| \mD_i^{(r, k+1)} \right\|^2 \\
    &\leq f (\mX^{(r, k)}) + \frac{\eta}{n} \mathbb{E}_{r, k} \sum_{i \in \mathcal{S}_{r}} \left\langle \nabla f (\mX^{(r, k)}) - \mG_i^{(r, k+1)}, \mD_i^{(r, k+1)} \right\rangle + \frac{\eta}{n} \mathbb{E}_{r, k} \sum_{i \in \mathcal{S}_{r}} \left\langle \mG_i^{(r, k+1)}, \mD_i^{(r, k+1)} \right\rangle + \frac{L S \eta^2}{2 n} \\
    &\leq f (\mX^{(r, k)}) + \frac{\eta}{n} \mathbb{E}_{r, k} \sum_{i \in \mathcal{S}_{r}} \left\| \nabla f (\mX^{(r, k)}) - \mG_i^{(r, k+1)} \right\|_\star + \frac{\eta}{n} \mathbb{E}_{r, k} \sum_{i \in \mathcal{S}_{r}} \underbrace{\left\langle \mG_i^{(r, k+1)}, \mD_i^{(r, k+1)} \right\rangle}_{\mathcal{T}_1} + \frac{L S \eta^2}{2 n},
\end{align*}
where we use \cref{lemma:smoothness}, $\| \mD_i^{(r, k+1)} \| \leq 1$, and the Cauchy-Schwarz inequality in the first, second, and third inequalities, and $\mG_i$ and $\mD_i$ are defined in \cref{sec:notation}.
Using \cref{lemma:lmo} and the triangle inequality, we have
\begin{align*}
    \mathcal{T}_1 = - \left\| \mG_i^{(r, k+1)} \right\|_\star
    \leq - \left\| \nabla f (\mX^{(r, k)}) \right\|_\star + \left\| \nabla f (\mX^{(r, k)}) - \mG_i^{(r, k+1)} \right\|_\star.
\end{align*}
Then, it holds
\begin{align*}
    \mathbb{E}_{r, k} f (\mX^{(r, k+1)})
    \leq f (\bar{\mX}^{(r, k)}) - \frac{\eta S}{n} \left\| \nabla f (\mX^{(r, k)}) \right\|_\star 
    + \frac{2 \eta}{n} \mathbb{E}_{r, k} \sum_{i \in \mathcal{S}_{r}} \underbrace{\left\| \nabla f (\mX^{(r, k)}) - \mG_i^{(r, k+1)} \right\|_\star}_{\mathcal{T}_2} 
    + \frac{L S \eta^2}{2 n}.
\end{align*}
When $r \geq 1$, we have
\begin{align*}
    \mathcal{T}_2 
    &= \left\| \nabla f (\mX^{(r, k)}) - \mM_i^{(r, k+1)} + \mC_i^{(r)} - \mC^{(r)} \right\|_\star \\
    &\leq \left\| \nabla f (\mX^{(r, k)}) - \nabla f (\mX^{(r-1, K-1)}) \right\|_\star
    + \left\| \nabla f (\mX^{(r-1, K-1)}) - \mC^{(r)} \right\|_\star
    + \left\| \mM_i^{(r, k+1)} - \mC_i^{(r)} \right\|_\star \\
    &\leq L \left\| \mX^{(r, k)} - \mX^{(r-1, K-1)} \right\|
    + \left\| \nabla f (\mX^{(r-1, K-1)}) - \mC^{(r)} \right\|_\star
    + \left\| \mM_i^{(r, k+1)} - \mC_i^{(r)} \right\|_\star \\
    &\leq \frac{L S K \eta}{n} + \left\| \nabla f (\mX^{(r-1, K-1)}) - \mC^{(r)} \right\|_\star
    + \left\| \mM_i^{(r, k+1)} - \mC_i^{(r)} \right\|_\star,
\end{align*}
where we use \cref{lemma:norm_of_update} in the last inequality.

When $r=0$, we have
\begin{align*}
    \mathcal{T}_2 
    &= \left\| \nabla f (\mX^{(0, k)}) - \mM_i^{(0, k+1)} + \mC_i^{(0)} - \mC^{(0)} \right\|_\star \\
    &\leq \left\| \nabla f (\mX^{(0, k)}) - \nabla f (\mX^{(0, 0)}) \right\|_\star
    + \left\| \nabla f (\mX^{(0, 0)}) - \mC^{(0)} \right\|_\star
    + \left\| \mM_i^{(0, k+1)} - \mC_i^{(0)} \right\|_\star \\
    &\leq L \left\| \mX^{(0, k)} - \mX^{(0, 0)} \right\|
    + \left\| \nabla f (\mX^{(0, 0)}) - \mC^{(0)} \right\|_\star
    + \left\| \mM_i^{(0, k+1)} - \mC_i^{(0)} \right\|_\star \\
    &\leq \frac{L S K \eta}{n} + \left\| \nabla f (\mX^{(0, 0)}) - \mC^{(0)} \right\|_\star
    + \left\| \mM_i^{(0, k+1)} - \mC_i^{(0)} \right\|_\star
\end{align*}
Then, using the following inequality:
\begin{align*}
    \mathbb{E} \left\| \nabla f (\mX^{(0, 0)}) - \mC^{(0)} \right\|_\star \leq \frac{\rho}{n} \sum_{i=1}^n \sqrt{\mathbb{E} \left\|  \nabla f_i (\mX^{(0, 0)}) - \mC^{(0)}_i \right\|^2_F} \leq \rho \sigma_0,
\end{align*}
we obtain the desired result.
\end{proof}

\begin{lemma}
Suppose that \cref{assumption:smoothness,assumption:stochastic_noise} hold,  $\mC_i^{(0)} \coloneqq \mM_i^{(0, 0)}$ and $\mC^{(0)} \coloneqq \tfrac{1}{n} \sum_{i=1}^n \mC_i^{(0)}$, 
\label{lemmq:approximation_quality_of_local_momentum}
\begin{align*}
    \frac{1}{n} \mathbb{E} \sum_{i \in \mathcal{S}_r} \left\| \nabla F_i (\mX_i^{(r-1, K-1)} ) - \mM_i^{(r, 0)} \right\|_\star
    &\leq \frac{S \rho \sigma}{n} \left( 1 - \frac{S \alpha}{n} \right)^{r-1}
    + \frac{S}{n} \alpha \rho \sqrt{K \sigma^2}
    +  6 L K \eta.
\end{align*}
\end{lemma}
\begin{proof}
Let $c_i (r-1)$ be the number of times that client $i$ has been sampled by round $r$.
We have
\begin{align*}
    \mM_i^{(r, 0)} = (1 - \alpha)^{c_i (r - 1) K} \mM_i^{(0, 0)} 
    + \alpha \sum_{r'=1}^{c_i (r-1)} \sum_{k'=0}^{K-1} (1 - \alpha)^{(c_i (r-1) - r')K + k} \nabla F_i (\mX_i^{(r', k')} ; \xi_i^{r', k'})
\end{align*}
To simplify the notation, we denote $r_i(r')$ by the number of rounds that client $i$ is sampled for the $r'$-th time.
Using this notation, we have
\begin{align*}
    &\mM_i^{(r, 0)} \\
    &= (1 - \alpha)^{c_i (r - 1) K} \mM_i^{(0, 0)} 
    + \alpha \sum_{c'=1}^{c_i (r-1) K} (1 - \alpha)^{c_i (r-1) K  - c'} \nabla F_i (\mX_i^{(r' (\lceil \frac{c'}{K} \rceil), c' - K \lceil \frac{c'}{K} \rceil)} ; \xi_i^{(r' (\lceil \frac{c'}{K} \rceil), c' - K \lceil \frac{c'}{K} \rceil)}) \\
    &= (1 - \alpha)^{c_i (r - 1) K} \left( \nabla F_i (\mX_i^{(0, 0)} ; \xi_i^{(0, 0)}) - \nabla f_i (\mX_i^{(0, 0)}) \right) \\ 
    &\quad + \alpha \sum_{c'=1}^{c_i (r-1) K} (1 - \alpha)^{c_i (r-1) K  - c'} \left(  \nabla F_i (\mX_i^{(r' (\lceil \frac{c'}{K} \rceil), c' - K \lceil \frac{c'}{K} \rceil)} ; \xi_i^{(r' (\lceil \frac{c'}{K} \rceil), c' - K \lceil \frac{c'}{K} \rceil)}) - \nabla f_i (\mX_i^{(r' (\lceil \frac{c'}{K} \rceil), c' - K \lceil \frac{c'}{K}) \rceil} \right)\\
    &\quad + \underbrace{(1 - \alpha)^{c_i (r - 1) K} \nabla f_i (\mX_i^{(0, 0)})
    + \alpha \sum_{c'=1}^{c_i (r-1) K} (1 - \alpha)^{c_i (r-1) K  - c'} \nabla f_i (\mX_i^{(r' (\lceil \frac{c'}{K} \rceil), c' - L \lceil \frac{c'}{K} \rceil)})}_{\mathcal{T}}.
\end{align*}
Using $\alpha (1 - \alpha)^m = (1 - \alpha)^m - (1 - \alpha)^{m+1}$, we have
\begin{align*}
    \mathcal{T}
    &= \nabla f_i (\mX_i^{(r_i (c_i (r-1)), K-1)}) \\
    &\quad + \sum_{c'=1}^{c_i (r-1) K - 1} (1 - \alpha)^{c_i (r-1)K - c'} \left( 
        \nabla f_i (\mX_i^{(r' (\lceil \frac{c'}{K} \rceil), c' - L \lceil \frac{c'}{K} \rceil)}) 
        - \nabla f_i (\mX_i^{(r' (\lceil \frac{c'+1}{K} \rceil), c' + 1 - K \lceil \frac{c'+1}{K} \rceil)})
    \right) \\
    &\quad + (1 - \alpha)^{c_i(r-1)K} \left( \nabla f_i (\mX_i^{(0,, 0)}) - \nabla f_i (\mX_i^{(r_i(1) , 0)}) \right).
\end{align*}
Thus, we have
\begin{align*}
    &\mathbb{E} \left\| \mM_i^{(r, 0)} - \nabla f_i (\mX_i^{(r-1, K-1)}) \right\|_\star \\
    &\leq \mathbb{E} (1 - \alpha)^{c_i (r-1) K} \rho \sigma \\
    &\quad + \alpha \mathbb{E} \left\| \sum_{c'=1}^{c_i (r-1) K} (1 - \alpha)^{c_i (r-1) K  - c'} \left(  \nabla F_i (\mX_i^{(r' (\lceil \frac{c'}{K} \rceil), c' - K \lceil \frac{c'}{K} \rceil)} ; \xi_i^{(r' (\lceil \frac{c'}{K} \rceil), c' - K \lceil \frac{c'}{K} \rceil)}) - \nabla f_i (\mX_i^{(r' (\lceil \frac{c'}{K} \rceil), c' - K \lceil \frac{c'}{K}) \rceil} \right) \right\|_\star  \\
    &\quad + \mathbb{E} \left\|  \nabla f_i (\mX_i^{(r_i (c_i (r-1)), K-1)})  - \nabla f_i (\mX_i^{(r-1, K-1)}) \right\|_\star \\
    &\quad +  \mathbb{E} \sum_{c'=1}^{c_i (r-1) K - 1} (1 - \alpha)^{c_i (r-1)K - c'} \left\|
        \nabla f_i (\mX_i^{(r' (\lceil \frac{c'}{K} \rceil), c' - L \lceil \frac{c'}{K} \rceil)}) 
        - \nabla f_i (\mX_i^{(r' (\lceil \frac{c'+1}{K} \rceil), c' + 1 - K \lceil \frac{c'+1}{K} \rceil)}) 
    \right\|_\star \\
    &\quad + \mathbb{E} (1 - \alpha)^{c_i(r-1)K} \left\| \nabla f_i (\mX_i^{(0,, 0)}) - \nabla f_i (\mX_i^{(r_i(1) , 0)}) \right\|_\star 
\end{align*}

Using \cref{assumption:smoothness}, we have
\begin{align*}
    &\mathbb{E} \left\| \mM_i^{(r, 0)} - \nabla f_i (\mX_i^{(r-1, K-1)}) \right\|_\star\\
    &\leq \underbrace{\mathbb{E} (1 - \alpha)^{c_i (r-1) K} \rho \sigma}_{\mathcal{T}_1}
    + \alpha \rho \sqrt{K \sigma^2} \\
    &\quad +  L \underbrace{\mathbb{E} \sum_{c'=1}^{c_i (r-1) K - 1} (1 - \alpha)^{c_i (r-1)K - c'} \left\|
        \mX_i^{(r' (\lceil \frac{c'}{K} \rceil), c' - L \lceil \frac{c'}{K} \rceil)}
        - \mX_i^{(r' (\lceil \frac{c'+1}{K} \rceil), c' + 1 - K \lceil \frac{c'+1}{K} \rceil)} 
    \right\|}_{\mathcal{T}_2} \\
    &\quad + L \underbrace{\mathbb{E} (1 - \alpha)^{c_i(r-1)K} \left\| \mX_i^{(0,, 0)} - \mX_i^{(r_i(1) , 0)} \right\|}_{\mathcal{T}_3} \\
    &\quad + L \underbrace{\mathbb{E} \left\| \mX_i^{(r_i (c_i (r-1)), K-1)}  - \mX_i^{(r-1, K-1)} \right\|}_{\mathcal{T}_4}.
\end{align*}

The quantity of $c_i (r -1)$ is the number of rounds in which client $i$ is sampled, which follows the binomial distribution.
\begin{align*}
    \mathcal{T}_1 
    &\leq \rho \sigma \sum_{c'=0}^{r-1} (1 - \alpha)^{K c'} \left( \frac{S}{n} \right)^{c'} \left( 1 - \frac{S}{n} \right)^{r - 1 - c'} \binom{r-1}{c'} \\
    &\leq \rho \sigma \sum_{c'=0}^{r-1} \left( (1 - \alpha) \frac{S}{n} \right)^{c'} \left( 1 - \frac{S}{n} \right)^{r - 1 - c'} \binom{r-1}{c'} \\
    &= \rho \sigma \left( 1 - \frac{S \alpha}{n} \right)^{r-1}.
\end{align*}

\begin{align*}
    \mathcal{T}_2
    &\leq \eta \mathbb{E} \sum_{c'=1}^{c_i (r-1) K - 1} (1 - \alpha)^{c_i (r-1)K - c'} \left( r_i \left( \lceil \frac{c'+1}{K} \rceil \right) -  r_i \left( \lceil \frac{c'}{K} \rceil \right) \right) \\
    &= \eta \mathbb{E} \sum_{c''=1}^{c_i (r-1) K - 1} (1 - \alpha)^{c''} \underbrace{\left( r_i \left( \lceil \frac{c'(r - 1) K - c'' +1}{K} \rceil \right) -  r_i \left( \lceil \frac{c'(r - 1) K - c'' +1}{K} \rceil \right) \right)}_{\mathcal{T}_5}.
\end{align*}
The quantity of $\mathcal{T}_5$ is the number of rounds from the time cline $i$ was sampled to the next sampling, which follows a geometric distribution with expectation $\frac{n}{S}$. Thus, we have
\begin{align*}
    \mathcal{T}_2 \leq \frac{K n \eta }{S}
\end{align*}

Using \cref{lemma:norm_of_update_of_client}, we have
\begin{align*}
    \mathcal{T}_3 
    \leq \mathbb{E} \left\| \mX_i^{(0, 0)} - \mX_i^{(r_i(1) , 0)} \right\|_\star 
    = \mathbb{E} \left\| \mX^{(0, 0)} - \mX^{(r_i(1), 0)} \right\|_\star 
    \leq \frac{\eta S}{n} K \mathbb{E} r_i (1),
\end{align*}
where we use \cref{lemma:norm_of_update} in the last inequality.
The quantity of $r_i (1)$ is the round in which client $i$ is sampled for the first time, which follows a geometric distribution. Thus, we have
\begin{align*}
    \mathcal{T}_3 
    \leq K \eta.
\end{align*}

Using \cref{lemma:norm_of_update_of_client}, we have
\begin{align*}
    \mathcal{T}_4
    = K \eta \left( \mathbb{E} \left( r - 1 - r_i (c_i (r - 1)) + 2 \right) \right).
\end{align*}
Since the quantity of $r_i (c_i (r - 1))$ is the rounds in which client $i$ is sampled for the last time, we have
\begin{align*}
    \mathbb{E} \left( r - 1 - r_i (c_i (r - 1)) \right) 
    = (r - 1) \left( 1 - \frac{S}{n} \right)^r
    + \sum_{r'=0}^r r' \left( 1 - \frac{S}{n} \right)^{r'} \frac{S}{n} 
    \leq \frac{2 n}{S}.
\end{align*}
Thus, it holds that
\begin{align*}
    \mathcal{T}_4 \leq \frac{4 K n \eta}{S}.
\end{align*}

By combining the above inequalities, we obtain the desired result.
\end{proof}

\begin{lemma}
\label{lemma:difference_of_momentum}
Suppose that \cref{assumption:smoothness,assumption:stochastic_noise} holds.
When $r \geq 1$, it holds that
\begin{align*}
    \frac{1}{n} \mathbb{E} \sum_{i \in \mathcal{S}_r} \left\| \mM_i^{(r, k+1)} - \mC_i^{(r)} \right\|_\star 
    &\leq
    2 \alpha \left( \frac{S}{n} \right) \rho \sqrt{K \sigma^2}
    + 9 K L \eta
    + \left( \frac{S}{n} \right) \rho \sigma_0 \left( 1 - \frac{S \alpha}{n} \right)^{r-1},
\end{align*}
where $\rho \coloneqq \sup_{\mX \in \mathcal{X}} \tfrac{\| \mX \|_\star}{\| \mX \|_F}$ and $\sigma_0^2 \coloneqq \frac{1}{n} \sum_{i=1}^n \mathbb{E} \| \nabla f_i (\mX_i^{(0)}) - \mC_i^{(0)} \|^2_F$.

Then, when $r=0$, we have
\begin{align*}
    \frac{1}{n} \mathbb{E} \sum_{i \in \mathcal{S}_0} \left\| \mM_i^{(0, k+1)} - \mC_i^{(0)} \right\|_\star
    &\leq \alpha \left( \frac{S}{n} \right) \rho \sqrt{K \sigma^2} + L K \eta + \left( \frac{S}{n} \right) \rho \sigma_0.
\end{align*}
\end{lemma}
\begin{proof}
We have
\begin{align*}
    \mM_i^{(r, k+1)} = (1 - \alpha)^{k+1} \mM_i^{(r, 0)} + \alpha \sum_{k'=0}^k (1 - \alpha)^{k - k'} \nabla F_i (\mX_i^{(r, k')} ; \xi_i^{(r, k')}).
\end{align*}
Since we have $\mC_i^{(r)} = \mM_i^{(r, 0)}$, we have
\begin{align*}
    \mathbb{E} \sum_{i \in \mathcal{S}_r} \left\| \mM_i^{(r, k+1)} - \mC_i^{(r)} \right\|_\star
    &= \alpha \mathbb{E} \sum_{i \in \mathcal{S}_r} \left\| \sum_{k'=0}^k (1 - \alpha)^{k - k'} \left( \nabla F_i (\mX_i^{(r, k')} ; \xi_i^{(r, k')}) - \mM_i^{(r, 0)} \right)  \right\|_\star.
\end{align*}

When $r \geq 1$, we have
\begin{align*}
    \mathbb{E} \sum_{i \in \mathcal{S}_r} \left\| \mM_i^{(r, k+1)} - \mC_i^{(r)} \right\|_\star
    &\leq \alpha \mathbb{E} \sum_{i \in \mathcal{S}_r} \left\| \sum_{k'=0}^k (1 - \alpha)^{k - k'} \left( \nabla F_i (\mX_i^{(r, k')} ; \xi_i^{(r, k')}) - \nabla f_i (\mX_i^{(r, k'}) \right)  \right\|_\star \\
    &\quad + \alpha \mathbb{E} \sum_{i \in \mathcal{S}_r} \left\| \sum_{k'=0}^k (1 - \alpha)^{k - k'} \left( \nabla f_i (\mX_i^{(r, k'}) - \nabla f_i (\mX_i^{(r-1, K-1}) \right)  \right\|_\star \\
    &\quad+ \alpha \mathbb{E} \sum_{i \in \mathcal{S}_r} \left\| \sum_{k'=0}^k (1 - \alpha)^{k - k'} \left(\nabla  f_i (\mX_i^{(r-1, K-1}) - \mM_i^{(r, 0)} \right)  \right\|_\star \\
    &\leq \alpha \mathbb{E} \sum_{i \in \mathcal{S}_r} \left\| \sum_{k'=0}^k (1 - \alpha)^{k - k'} \left( \nabla F_i (\mX_i^{(r, k')} ; \xi_i^{(r, k')}) - \nabla f_i (\mX_i^{(r, k'}) \right)  \right\|_\star \\
    &\quad + \alpha \mathbb{E} \sum_{i \in \mathcal{S}_r} \left\| \sum_{k'=0}^k (1 - \alpha)^{k - k'} \left( \nabla f_i (\mX_i^{(r, k'}) - \nabla f_i (\mX_i^{(r-1, K-1}) \right)  \right\|_\star \\
    &\quad+  \mathbb{E} \sum_{i \in \mathcal{S}_r} \left\| \nabla  f_i (\mX_i^{(r-1, K-1)}) - \mM_i^{(r, 0)}  \right\|_\star.
\end{align*}

The first term is bounded from above as follows:
\begin{align*}
    &\mathbb{E} \left\| \sum_{k'=0}^k (1 - \alpha)^{k - k'} \left( \nabla F_i (\mX_i^{(r, k')} ; \xi_i^{(r, k')}) - \nabla f_i (\mX_i^{(r, k'}) \right)  \right\|_\star \\
    &\leq \sqrt{\mathbb{E} \left\| \sum_{k'=0}^k (1 - \alpha)^{k - k'} \left( \nabla F_i (\mX_i^{(r, k')} ; \xi_i^{(r, k')}) - \nabla f_i (\mX_i^{(r, k'}) \right)  \right\|^2_\star} \\
    &= \sqrt{\sum_{k'=0}^k (1 - \alpha)^{2(k - k')} \mathbb{E} \left\| \nabla F_i (\mX_i^{(r, k')} ; \xi_i^{(r, k')}) - \nabla f_i (\mX_i^{(r, k'})  \right\|^2_\star} \\
    &\leq \rho \sqrt{K \sigma^2},
\end{align*}
where we used Jensen's inequality in the first inequality.
The second term is bounded as follows:
\begin{align*}
    &\left\| \sum_{k'=0}^k (1 - \alpha)^{k - k'} \left( \nabla f_i (\mX_i^{(r, k')}) - \nabla f_i (\mX_i^{(r-1, K-1}) \right)  \right\|_\star  \\
    &\leq \sum_{k'=0}^k (1 - \alpha)^{k - k'}  \left\| \nabla f_i (\mX_i^{(r, k')}) - \nabla
    f_i (\mX_i^{(r-1, K-1})  \right\|_\star \\
    &\leq L \sum_{k'=0}^k (1 - \alpha)^{k - k'}  \left\| \mX_i^{(r, k')} - \mX_i^{(r-1, K-1)}  \right\| \\
    &\leq \frac{3 L K \eta}{\alpha},
\end{align*}
where we use \cref{lemma:norm_of_update_of_client} in the last inequality.
Then, using \cref{lemmq:approximation_quality_of_local_momentum}, we obtain the desired result when $r \geq 1$.

When $r=0$, we have
\begin{align*}
    \mathbb{E} \sum_{i \in \mathcal{S}_0} \left\| \mM_i^{(0, k+1)} - \mC_i^{(0)} \right\|_\star
    &= \alpha \mathbb{E} \sum_{i \in \mathcal{S}_r} \left\| \sum_{k'=0}^k (1 - \alpha)^{k - k'} \left( \nabla F_i (\mX_i^{(r, k')} ; \xi_i^{(r, k')}) - \mM_i^{(r, 0)} \right)  \right\|_\star \\
    &\leq \alpha \mathbb{E} \sum_{i \in \mathcal{S}_0} \left\| \sum_{k'=0}^k (1 - \alpha)^{k - k'} \left( \nabla F_i (\mX_i^{(0, k')} ; \xi_i^{(0, k')}) - \nabla f_i (\mX_i^{(0, k')}) \right) \right\|_\star \\
    &\quad + \alpha \mathbb{E} \sum_{i \in \mathcal{S}_0} \left\| \sum_{k'=0}^k (1 - \alpha)^{k - k'} \left( \nabla f_i (\mX_i^{(0, k')}) - \nabla f_i (\mX_i^{(0, 0)}) \right) \right\|_\star \\
    &\quad + \alpha \mathbb{E} \sum_{i \in \mathcal{S}_0} \left\| \sum_{k'=0}^k (1 - \alpha)^{k - k'} \left( \nabla f_i (\mX_i^{(0, 0)}) - \mM_i^{(0, 0)} \right) \right\|_\star \\
    &\leq \alpha S \rho \sqrt{K \sigma^2} + L K S \eta + S \rho \mathbb{E} \left\| \nabla f_i (\mX_i^{(0, 0)}) - \mC_i^{(0)} \right\|_F,
\end{align*}
where we use \cref{assumption:smoothness,assumption:stochastic_noise} and $\mC_i^{(0)} = \mM_i^{(r, 0)}$ in the last inequality.
Dividing both sides by $n$, we obtain the desired result.
\end{proof}

\begin{lemma}
\label{lemma:approximation_of_global_function}
Suppose that \cref{assumption:smoothness,assumption:stochastic_noise} hold,  $\mC_i^{(0)} \coloneqq \mM_i^{(0, 0)}$ and $\mC^{(0)} \coloneqq \tfrac{1}{n} \sum_{i=1}^n \mC_i^{(0)}$, we have
\begin{align*}
    \mathbb{E} \left\| \nabla f (\mX^{(r-1, K-1)}) - \mC^{(r)} \right\|_\star 
    &\leq \frac{\rho S}{n} \sqrt{\frac{\alpha \sigma^2}{S}}
    + \frac{4 n L \eta}{\alpha S}
    + \frac{6 L K n \eta}{S}
    + \rho \sigma \left( 1 - \frac{S\alpha}{n} \right)^{r-1}.
\end{align*}
\end{lemma}
\begin{proof}
Let $c_i (r-1)$ be the number of times that client $i$ has been sampled by round $r$.
We denote $r_i(r')$ by the number of rounds that client $i$ is sampled for the $r'$-th time.
Using this notation, we have
\begin{align*}
    \mC^{(r)} = \frac{1}{n} \sum_{i=1}^n \mM_i^{(r_i (c_i(r-1)), K)}.
\end{align*}
Then, we have
\begin{align*}
    \mM_i^{(r-1, K-1)} 
    &= (1 - \alpha)^{c_i (r - 1) K} \mM_i^{(0, 0)} 
    + \alpha \sum_{r'=1}^{c_i (r-1)} \sum_{k'=0}^{K-1} (1 - \alpha)^{(c_i (r-1) - r')K + k} \nabla F_i (\mX_i^{(r', k')} ; \xi_i^{r', k'}) \\
    &= (1 - \alpha)^{c_i (r - 1) K} \left( \mM_i^{(0, 0)} - \nabla f_i (\mX^{(0, 0)}) \right) \\
    &\quad + \underbrace{\alpha \sum_{r'=1}^{c_i (r-1)} \sum_{k'=0}^{K-1} (1 - \alpha)^{(c_i (r-1) - r')K + k} \left( \nabla F_i (\mX_i^{(r', k')} ; \xi_i^{r', k'}) - \nabla f_i (\mX_i^{(r', k')}) \right)}_{\mathcal{T}_1} \\
    &\quad + \underbrace{(1 - \alpha)^{c_i (r - 1) K} \nabla f_i (\mX^{(0, 0)})
     + \alpha \sum_{r'=1}^{c_i (r-1)} \sum_{k'=0}^{K-1} (1 - \alpha)^{(c_i (r-1) - r')K + k} \nabla f_i (\mX_i^{(r', k')})}_{\mathcal{T}_2}.
\end{align*}
We can rewrite $\mathcal{T}_1$ and $\mathcal{T}_2$ as follows:
\begin{align*}
    \mathcal{T}_1 
    &= \alpha \sum_{r'=0}^{r-1} \sum_{k'=1}^K \mathbbm{1}_{i \in \mathcal{S}_{r'}} (1 - \alpha)^{(c_i (r-1) - r' + 1)K - k'} \left( \nabla F_i (\mX_i^{(r', k')} ; \xi_i^{(r', k')}) - \nabla f_i (\mX_i^{(r', k')} )\right).
\end{align*}

\begin{align*}
    \mathcal{T}_2
    &= (1 - \alpha)^{c_i (r - 1) K} \nabla f_i (\mX^{(0, 0)})
     + \alpha \sum_{c'=1}^{c_i (r-1) K} (1 - \alpha)^{c_i (r-1)K - c'} \nabla f_i (\mX_i^{(r_i (\lceil \frac{c'}{K} \rceil), c' - K \lceil \frac{c'}{K} \rceil)}) \\
    &= \nabla f_i (\mX_i^{(r_i (c_i (r-1)), K-1)}) \\
    &\quad + \alpha \sum_{c'=1}^{c_i (r-1) K - 1} (1 - \alpha)^{c_i (r-1)K - c'} \left( \nabla f_i (\mX_i^{(r_i (\lceil \frac{c'}{K} \rceil), c' - K \lceil \frac{c'}{K} \rceil)}) - \nabla f_i (\mX_i^{(r_i (\lceil \frac{c'+1}{K} \rceil), c' + 1 - K \lceil \frac{c'+1}{K} \rceil)}) \right) \\
    &\quad + (1 - \alpha)^{c_i (r-1)K} \left( \nabla f_i (\mX^{(0, 0}) - \nabla f_i (\mX_i^{(r_i (1), 0)})\right).
\end{align*}
Thus, we have
\begin{align*}
    &\mathbb{E} \left\| \mM^{(r-1, K-1)} - \nabla f (\mX^{(r-1, K-1)}) \right\|_\star \\
    &\leq \underbrace{\mathbb{E} (1 - \alpha)^{c_i (r-1)K} \left\| \frac{1}{n} \sum_{i=1}^n \left( \nabla f_i (\mX^{(0, 0)} ; \xi_i^{(0, 0)}) - \nabla f_i (\mX^{(0, 0)}) \right) \right\|_\star}_{\mathcal{T}_3} \\
    &\quad + \alpha \underbrace{\mathbb{E} \left\| \frac{1}{n} \sum_{i=1}^n \sum_{r'=0}^{r-1} \sum_{k'=1}^K \mathbbm{1}_{i \in \mathcal{S}_{r'}} (1 - \alpha)^{(c_i (r-1) - r' + 1)K - k'} \left( \nabla F_i (\mX_i^{(r', k')} ; \xi_i^{(r', k')}) - \nabla f_i (\mX_i^{(r', k')} )\right) \right\|_\star}_{\mathcal{T}_4}  \\
    &\quad + \frac{1}{n} \sum_{i=1}^n \underbrace{\mathbb{E} \left\| \nabla f_i (\mX_i^{(r_i (c_i (r-1)), K-1)}) - \nabla f_i (\mX^{(r-1, K-1}) \right\|_\star}_{\mathcal{T}_5} \\
    &\quad + \frac{1}{n} \sum_{i=1}^n \underbrace{\mathbb{E} \sum_{c'=1}^{c_i (r-1) K - 1} (1 - \alpha)^{c_i (r-1)K - c'} \left\| \nabla f_i (\mX_i^{(r_i (\lceil \frac{c'}{K} \rceil), c' - K \lceil \frac{c'}{K} \rceil)}) - \nabla f_i (\mX_i^{(r_i (\lceil \frac{c'+1}{K} \rceil), c' + 1 - K \lceil \frac{c'+1}{K} \rceil)}) \right\|_\star}_{\mathcal{T}_6} \\
    &\quad + \frac{1}{n} \sum_{i=1}^n \underbrace{\mathbb{E} (1 - \alpha)^{c_i (r-1)K} \left\| \nabla f_i (\mX^{(0, 0}) - \nabla f_i (\mX_i^{(r_i (1), 0)})\right\|_\star}_{\mathcal{T}_7}.
\end{align*}

\begin{align*}
    \mathcal{T}_3 \leq \mathbb{E} (1 - \alpha)^{c_i (r-1)K} \rho \sigma
\end{align*}
The quantity of $c_i (r -1)$ is the number of rounds in which client $i$ is sampled, which follows the binomial distribution.
\begin{align*}
    \mathcal{T}_3 
    &\leq \rho \sigma \sum_{c'=0}^{r-1} (1 - \alpha)^{K c'} \left( \frac{S}{n} \right)^{c'} \left( 1 - \frac{S}{n} \right)^{r - 1 - c'} \binom{r-1}{c'} \\
    &\leq \rho \sigma \sum_{c'=0}^{r-1} \left( (1 - \alpha) \frac{S}{n} \right)^{c'} \left( 1 - \frac{S}{n} \right)^{r - 1 - c'} \binom{r-1}{c'} \\
    &= \rho \sigma \left( 1 - \frac{S \alpha}{n} \right)^{r-1}
\end{align*}

\begin{align*}
    \mathcal{T}_4 
    &=  \frac{1}{n} \mathbb{E} \left\| \sum_{r'=0}^{r-1} \sum_{k'=1}^K \sum_{i \in \mathcal{S}_{r'}} (1 - \alpha)^{(c_i (r-1) - r' + 1)K - k'} \left( \nabla F_i (\mX_i^{(r', k')} ; \xi_i^{(r', k')}) - \nabla f_i (\mX_i^{(r', k')} )\right) \right\|_\star \\
    &\leq \frac{\rho S}{n} \sqrt{\frac{\sigma^2}{S (1 - (1-\alpha)^2)}}.
\end{align*}

\begin{align*}
    \mathcal{T}_5
    &= \mathbb{E} \left\| \nabla f_i (\mX_i^{(r_i (c_i (r-1)), K-1)}) - \nabla f_i (\mX^{(r-1, K-1)}) \right\|_\star \\
    &\leq L \mathbb{E} \left\| \mX_i^{(r_i (c_i (r-1)), K-1)} - \mX^{(r-1, K-1)} \right\| \\
    &\leq L \mathbb{E} \left\| \mX_i^{(r_i (c_i (r-1)), K-1)} - \mX_i^{(r_i ( c_i (r-1)), 0)} \right\| + L \mathbb{E} \left\| \mX^{(r_i ( c_i (r-1)), 0)} - \mX^{(r-1, K-1)} \right\| \\
    &\leq L \eta (K - 1) + \frac{L S \eta}{n} \mathbb{E} (\left( (r - 1 - r_i (c_i (r-1)) \right)K + K-1 ).
\end{align*}
Since the quantity of $r_i (c_i (r - 1))$ is the rounds in which client $i$ is sampled for the last time, we have
\begin{align*}
    \mathbb{E} \left( r - 1 - r_i (c_i (r - 1)) \right) 
    = (r - 1) \left( 1 - \frac{S}{n} \right)^r
    + \sum_{r'=0}^r r' \left( 1 - \frac{S}{n} \right)^{r'} \frac{S}{n} 
    \leq \frac{2 n}{S}.
\end{align*}
Thus, we have
\begin{align*}
    \mathcal{T}_5
    &\leq 2 L K \eta.
\end{align*}

\begin{align*}
    \mathcal{T}_6
    &\leq L \mathbb{E} \sum_{c'=1}^{c_i (r-1) K - 1} (1 - \alpha)^{c_i (r-1)K - c'} \left\| \mX_i^{(r_i (\lceil \frac{c'}{K} \rceil), c' - K \lceil \frac{c'}{K} \rceil)} - \mX_i^{(r_i (\lceil \frac{c'+1}{K} \rceil), c' + 1 - K \lceil \frac{c'+1}{K} \rceil)} \right\| \\
    &= L \mathbb{E} \sum_{c''=1}^{c_i (r-1)} \sum_{k'=0}^{K-2} (1 - \alpha)^{(c_i (r-1) - c'') K - k'} \left\| \mX_i^{(r_i (c''), k'+1)} - \mX_i^{(r_i (c''), k')} \right\| \\
    &\quad + L \mathbb{E} \sum_{c''=1}^{c_i (r-1)} (1 - \alpha)^{( c_i (r-1) - c'' ) K +1} \left\| \mX_i^{(r_i (c'' - 1), 0)} - \mX_i^{(r_i ( c'' ), K - 1)} \right\| \\
    &\leq + L \mathbb{E} \sum_{c''=1}^{c_i (r-1)} \sum_{k'=0}^{K-2} (1 - \alpha)^{(c_i (r-1) - c'') K - k'} \\
    &\quad + L \eta \mathbb{E} \sum_{c''=1}^{c_i (r-1) - 1} (1 - \alpha)^{( c_i (r-1) - c'' ) K +1} \left( r_i (c''+1)  - r_i (c'') + 2 \right) \\
    &\leq \frac{L \eta}{\alpha} + L \eta \mathbb{E} \sum_{c''=1}^{c_i (r-1) - 1} (1 - \alpha)^{( c_i (r-1) - c'' ) K +1} \left( r_i (c''+1)  - r_i (c'') + 2 \right).
\end{align*}
The quantity $r_i (c''+1) - r_i (c'')$ follows the geometric distribution, which has the expectation of $\frac{n}{S}$.
Using \cref{lemma:sum_of_geometric_series}, we obtain
\begin{align*}
    \mathcal{T}_6 \leq \frac{4 n L \eta}{\alpha S} + \frac{3 n L K \eta}{S}.
\end{align*}

\begin{align*}
    \mathcal{T}_7 &\leq L \mathbb{E} (1 - \alpha)^{c_i (r-1)K} \left\| \mX_i^{(0, 0)} - \mX_i^{(r_i (1), 0)} \right\|
    \leq L K \eta \mathbb{E} ( r_i (1) + 2),
\end{align*}
where we used \cref{lemma:norm_of_update_of_client} in the last inequality.
The quantity of $r_i (1)$ is the round in which client $i$ is sampled for the first time, which follows a geometric distribution. Thus, we have
\begin{align*}
    \mathcal{T}_7 &\leq \frac{3 n L K \eta}{S} 
\end{align*}
\end{proof}

\begin{lemma}
\label{lemma:convergence_rate}
Suppose that \cref{assumption:smoothness,assumption:stochastic_noise} hold,  $\mC_i^{(0)} \coloneqq \mM_i^{(0, 0)}$ and $\mC^{(0)} \coloneqq \tfrac{1}{n} \sum_{i=1}^n \mC_i^{(0)}$, there exists $\eta$ and $\alpha$ such that we have
\begin{align*}
    \frac{1}{R K} \sum_{r=0}^{R-1} \sum_{k=0}^{K-1} \mathbb{E} \left\| \nabla f (\mX^{(r, k)}) \right\|_\star 
    &\leq \mathcal{O} \left( 
        \left( \frac{L r_0 \rho^2 \sigma^2}{S R K} \right)^\frac{1}{4}
        + \left( \left( \frac{n}{S} \right)^2 \frac{L r_0 \rho \sigma}{R \sqrt{K}} \right)^\frac{1}{3}
        + \left( \frac{L r_0}{R} \left( \frac{n}{S}  \right)^2 \right)^\frac{1}{2}
        + \frac{\rho \sigma_0}{R} \left( \frac{n}{S} \right) \right. \\
        &\quad \left. + \rho \sigma_0 \left( \frac{\rho^2 \sigma^2 K S}{L r_0 R n^2} \right)^{\frac{1}{2}}
        + \rho \sigma_0 \left( \left( \frac{n}{S} \right) \frac{\rho^2 \sigma^2 K^2}{L r_0 R^2}\right)^\frac{1}{3}
    \right)
\end{align*}
\end{lemma}
\begin{proof}
Combining \cref{lemma:descent_lemma,lemma:approximation_of_global_function,lemma:difference_of_momentum}, when $r \geq 1$, it holds
\begin{align*}
    \mathbb{E} f (\mX^{(r, k+1)})
    &\leq \mathbb{E} f (\mX^{(r, k)}) - \eta \left( \frac{S}{n} \right) \left\| \nabla f (\mX^{(r, k)}) \right\|_\star \\
    &\quad + 2 \rho \eta \left( \frac{S}{n} \right)^2 \sqrt{\frac{\alpha \sigma^2}{S}}
    + \frac{8 L \eta^2}{\alpha} 
    + 33 L K \eta^2 
    + 4 \alpha \rho \eta \left( \frac{S}{n} \right) \sqrt{K \sigma^2} \\
    &\quad + 4 \rho \eta \left( \frac{S}{n} \right) \sigma_0 \left( 1 - \frac{S\alpha}{n} \right)^{r-1}.
\end{align*}
When $r=0$, we have
\begin{align*}
    \mathbb{E} f (\mX^{(0, k+1)})
    &\leq \mathbb{E} f (\mX^{(0, k)}) - \frac{\eta S}{n} \left\| \nabla f (\mX^{(0, k)}) \right\|_\star  \\
    &\quad + 5 L K \eta^2
    + 2 \alpha \rho \eta \left( \frac{S}{n} \right) \sqrt{K \sigma^2}
    + 4 \rho \eta \left( \frac{S}{n} \right) \sigma_0.
\end{align*}

Summing up the above two inequalities, we obtain
\begin{align*}
    \frac{1}{R K} \sum_{r=0}^{R-1} \sum_{k=0}^{K-1} \mathbb{E} \left\| \nabla f (\mX^{(r, k)}) \right\|_\star
    &\leq \left( \frac{n}{S} \right) \frac{r_0}{R K \eta} 
    + 2 \rho \left( \frac{S}{n} \right) \sqrt{\frac{\alpha \sigma^2}{S}}
    + \frac{8 L \eta}{\alpha} \left( \frac{n}{S} \right)
    + 33 L K \left( \frac{n}{S} \right) \eta \\
    &\quad + 4 \alpha \rho \sqrt{K \sigma^2} 
    + \frac{4 \rho \sigma_0}{R}  \sum_{r=1}^{R-1} \left( 1 - \frac{S\alpha}{n} \right)^{r-1}
    + \frac{4 \rho \sigma_0}{R} \\
    &\leq \left( \frac{n}{S} \right) \frac{r_0}{R K \eta} 
    + 2 \rho \left( \frac{S}{n} \right) \sqrt{\frac{\alpha \sigma^2}{S}}
    + \frac{8 L \eta}{\alpha} \left( \frac{n}{S} \right)
    + 33 L K \left( \frac{n}{S} \right) \eta \\
    &\quad + 4 \alpha \rho \sqrt{K \sigma^2} 
    + \frac{8 \rho \sigma_0}{R \alpha} \left( \frac{n}{S}  \right).
\end{align*}
Then, using the following hyperparameters
\begin{align*}
    \eta &= \min \left\{ \sqrt{\frac{\alpha r_0}{8 L R K}},
    \frac{1}{K}  \sqrt{\frac{r_0}{33 L R}}
    \right\}, \\
    \alpha &= \min \left\{ 1,  \left(\frac{n}{S}\right)^2 \sqrt{\frac{8 L r_0 S}{R L \rho^2 \sigma^2}},
    \left( \frac{2 L r_0}{R K^2 \rho^2 \sigma^2} \left(\frac{n}{S} \right)^2 \right)^{\frac{1}{3}}
    \right\},
\end{align*}
we obtain the desired result.
\end{proof}

\newpage

\section{Proof of \cref{theorem:convergence_analysis_with_ns_iteration}}
\label{sec:proof_of_theorem_with_inexact_lmo}

\begin{lemma}
\label{lemma:inexact_lmo2}
Let $\mG$ and $- \mG^{(T)}$ be the input and output of \cref{algorithm:ns_iteration} with $a=\frac{15}{8}, b=-\frac{5}{4}$, and $c=\frac{3}{8}$.
For any number of iterations $T$, we have
\begin{align*}
    \langle \mG, - \mG^{(T)} \rangle \leq - \| \mG \|_p,
\end{align*}
where $p$ is defined as follows:
\begin{align*}
    p &\coloneqq 1 + \frac{\log \left( 1 - (1 - \kappa)^{{1.5}^T} \right)}{\log \kappa}, \\
    \kappa &\coloneqq \min_i \frac{s_i}{\sqrt{\sum_{j} s_j^2}} \; (> 0),
\end{align*}
and $s_i$ is the non-zero singular value of $\mG$.
\end{lemma}
\begin{proof}
Let the singular value decomposition of $\mG$ be $\mU \Sigma \mV$.
Then, the output $- \mG^{(T)}$ can be written as follows:
\begin{align*}
    - \mG^{(T)} = - \mU \Sigma^{(T)} \mV,
\end{align*}
where $\Sigma^{(T)}$ is defined as follows:
\begin{align*}
    \Sigma^{(T)}_{ii} \coloneqq \underbrace{\phi \left(\phi \left( \cdots \phi \left(\frac{\Sigma_{ii}}{\| \mG \|_F} \right) \right) \right)}_{T \text{times}}.
\end{align*}
Since $\phi (x) > x$, we have
\begin{align*}
    \Sigma^{(T)}_{ii} \geq \frac{\Sigma_{ii}}{\| \mG \|_F}.
\end{align*}
Using the above inequality, we have
\begin{align*}
    \langle \mG, - \mG^{(T)} \rangle 
    &= - \left\langle \mU \Sigma \mV, \mU \Sigma^{(T)} \mV \right\rangle \\
    &= - \sum_{i} \Sigma_{ii} \left( 1 - \left( 1 -\Sigma^{(T)}_{ii} \right) \right).
\end{align*}
When $a=\frac{15}{8}, b=-\frac{5}{4}$ and $c=\frac{3}{8}$, we have
\begin{align*}
    0 \leq 1 - \phi(x) = (1 - x)^2 (- \frac{3}{8} x^3 - \frac{3}{4} x^2 + \frac{1}{8} x + 1) \leq (1 - x)^{1.5}
\end{align*}
Thus, it holds that
\begin{align*}
    1 - \Sigma_{ii}^{(T)} \leq \left( 1 - \frac{\Sigma_{ii}}{\| \mG \|_F} \right)^{{1.5}^T} \leq \left( 1 - \frac{\Sigma_{ii}}{\left( \sum_{j} \Sigma_{jj}^{p} \right)^{\frac{1}{p}}} \right)^{{1.5}^T},
\end{align*}
for any $1 \leq p \leq 2$.
Using the above inequality and the definition of $p$ and $\kappa$ we get
\begin{align*}
    \langle \mG, - \mG^{(T)} \rangle 
    &\leq - \left( \sum_{i} \Sigma_{ii}^p \right)^{\frac{1}{p}}.
\end{align*}
\end{proof}

\begin{lemma}
\label{lemma:descent_lemma_for_muon}
Suppose that \cref{assumption:smoothness,assumption:stochastic_noise} hold.
Then, when $r \geq 1$, we have
\begin{align*}
    \mathbb{E} f (\mX^{(r, k+1)})
    &\leq \mathbb{E} f (\mX^{(r, k)}) - \frac{\eta S}{n} \left\| \nabla f (\mX^{(r, k)}) \right\|_p 
    + 2 L K \left( \frac{S}{n} \right)^2 \eta^2 \\
    &\quad + 2 \left( \frac{S}{n} \right) \eta \mathbb{E} \left\| \nabla f (\mX^{(r-1, K-1)}) - \mC^{(r)} \right\|_\text{trace} 
    + \frac{2 \eta}{n} \mathbb{E} \sum_{i \in \mathcal{S}_r} \left\| \mM_i^{(r, k+1)} - \mC_i^{(r)} \right\|_\text{trace}
    + \frac{L}{2} \left( \frac{S}{n} \right) \eta^2,
\end{align*}
where $p$ is defined as follows:
\begin{align*}
    p &\coloneqq 1 + \frac{\log \left( 1 - (1 - \kappa)^{{1.5}^T} \right)}{\log \kappa}, \\
    \kappa &\coloneqq \min_{j,i,r,k} \frac{s_{j,i,r,k}}{\sqrt{\sum_{j'} s_{j',i,r,k}^2}} \; (> 0),
\end{align*}
and $\{ s_{j,r,k} \}_j$ are non-zero singular values of $\mM_i^{(r, k+1)} - \mC_i^{(r)} + \mC^{(r)}$. 

When $r=0$, we have
\begin{align*}
    \mathbb{E} f (\mX^{(0, k+1)})
    &\leq \mathbb{E} f (\mX^{(0, k)}) - \frac{\eta S}{n} \left\| \nabla f (\mX^{(0, k)}) \right\|_p
    + 2 L K \left( \frac{S}{n} \right)^2 \eta^2  \\
    &\quad + \frac{2 \eta}{n} \mathbb{E} \sum_{i \in \mathcal{S}_0} \left\| \mM_i^{(0, k+1)} - \mC_i^{(0)} \right\|_\text{trace}
    + \frac{L}{2} \left( \frac{S}{n} \right) \eta^2 
    + 2 \left( \frac{S}{n} \right) \rho \sigma_0 \eta.
\end{align*}
\end{lemma}
\begin{proof}
We have
\begin{align*}
    &\mathbb{E}_{r, k} f (\mX^{(r, k+1)}) \\
    &= \mathbb{E}_{r, k} f \left( \mX^{(r, k)} + \frac{\eta}{n} \sum_{i \in \mathcal{S}_r} \mD_i^{(r, k)} \right) \\
    &\leq f (\mX^{(r, k)}) + \frac{\eta}{n} \mathbb{E}_{r, k} \sum_{i \in \mathcal{S}_r} \left\langle \nabla f (\mX^{(r, k)}), \mD_i^{(r, k)} \right\rangle + \frac{L \eta^2}{2 n} \mathbb{E}_{r, k} \sum_{i \in \mathcal{S}_r} \left\| \mD_i^{(r, k+1)} \right\|^2_\text{sp} \\
    &\leq f (\mX^{(r, k)}) + \frac{\eta}{n} \mathbb{E}_{r, k} \sum_{i \in \mathcal{S}_{r}} \left\langle \nabla f (\mX^{(r, k)}) - \mG_i^{(r, k+1)}, \mD_i^{(r, k+1)} \right\rangle + \frac{\eta}{n} \mathbb{E}_{r, k} \sum_{i \in \mathcal{S}_{r}} \left\langle \mG_i^{(r, k+1)}, \mD_i^{(r, k+1)} \right\rangle + \frac{L S \eta^2}{2 n} \\
    &\leq f (\mX^{(r, k)}) + \frac{\eta}{n} \mathbb{E}_{r, k} \sum_{i \in \mathcal{S}_{r}} \left\| \nabla f (\mX^{(r, k)}) - \mG_i^{(r, k+1)} \right\|_\text{trace} + \frac{\eta}{n} \mathbb{E}_{r, k} \sum_{i \in \mathcal{S}_{r}} \underbrace{\left\langle \mG_i^{(r, k+1)}, \mD_i^{(r, k+1)} \right\rangle}_{\mathcal{T}_1} + \frac{L S \eta^2}{2 n},
\end{align*}
where we use \cref{lemma:smoothness}, $\| \mD_i^{(r, k+1)} \|_\text{sp} \leq 1$, and the \cref{lemma:inner_product} in the first, second, and third inequalities.
Using \cref{lemma:inexact_lmo2}, the definition of $p$, and the triangle inequality, we have
\begin{align*}
    \mathcal{T}_1 
    &\leq - \left\| \mG_i^{(r, k+1)} \right\|_p \\
    &\leq - \left\| \nabla f (\mX^{(r, k)}) \right\|_p + \left\| \nabla f (\mX^{(r, k)}) - \mG_i^{(r, k+1)} \right\|_p \\
    &\leq - \left\| \nabla f (\mX^{(r, k)}) \right\|_p + \left\| \nabla f (\mX^{(r, k)}) - \mG_i^{(r, k+1)} \right\|_\text{trace},
\end{align*}
where we use the fact that $p \geq 1$ and $\| \mA \|_p \leq \| \mA \|_\text{trace}$ for any $\mA$.
Then, it holds
\begin{align*}
    \mathbb{E}_{r, k} f (\mX^{(r, k+1)})
    \leq f (\bar{\mX}^{(r, k)}) - \frac{\eta S}{n} \left\| \nabla f (\mX^{(r, k)}) \right\|_p 
    + \frac{2 \eta}{n} \mathbb{E}_{r, k} \sum_{i \in \mathcal{S}_{r}} \underbrace{\left\| \nabla f (\mX^{(r, k)}) - \mG_i^{(r, k+1)} \right\|_\text{trace}}_{\mathcal{T}_2} 
    + \frac{L S \eta^2}{2 n}.
\end{align*}

When $r \geq 1$, we have
\begin{align*}
    \mathcal{T}_2 
    &= \left\| \nabla f (\mX^{(r, k)}) - \mM_i^{(r, k+1)} + \mC_i^{(r)} - \mC^{(r)} \right\|_\text{trace} \\
    &\leq \left\| \nabla f (\mX^{(r, k)}) - \nabla f (\mX^{(r-1, K-1)}) \right\|_\text{trace}
    + \left\| \nabla f (\mX^{(r-1, K-1)}) - \mC^{(r)} \right\|_\text{trace}
    + \left\| \mM_i^{(r, k+1)} - \mC_i^{(r)} \right\|_\text{trace} \\
    &\leq L \left\| \mX^{(r, k)} - \mX^{(r-1, K-1)} \right\|_\text{sp}
    + \left\| \nabla f (\mX^{(r-1, K-1)}) - \mC^{(r)} \right\|_\text{trace}
    + \left\| \mM_i^{(r, k+1)} - \mC_i^{(r)} \right\|_\text{trace}\\
    &\leq \frac{L S K \eta}{n} + \left\| \nabla f (\mX^{(r-1, K-1)}) - \mC^{(r)} \right\|_\text{trace}
    + \left\| \mM_i^{(r, k+1)} - \mC_i^{(r)} \right\|_\text{trace},
\end{align*}
where we use \cref{lemma:norm_of_update} in the last inequality.

When $r=0$, we have
\begin{align*}
    \mathcal{T}_2 
    &= \left\| \nabla f (\mX^{(0, k)}) - \mM_i^{(0, k+1)} + \mC_i^{(0)} - \mC^{(0)} \right\|_\text{trace} \\
    &\leq \left\| \nabla f (\mX^{(0, k)}) - \nabla f (\mX^{(0, 0)}) \right\|_\text{trace}
    + \left\| \nabla f (\mX^{(0, 0)}) - \mC^{(0)} \right\|_\text{trace}
    + \left\| \mM_i^{(0, k+1)} - \mC_i^{(0)} \right\|_\text{trace} \\
    &\leq L \left\| \mX^{(0, k)} - \mX^{(0, 0)} \right\|_\text{sp}
    + \left\| \nabla f (\mX^{(0, 0)}) - \mC^{(0)} \right\|_\text{trace}
    + \left\| \mM_i^{(0, k+1)} - \mC_i^{(0)} \right\|_\text{trace} \\
    &\leq \frac{L S K \eta}{n} + \left\| \nabla f (\mX^{(0, 0)}) - \mC^{(0)} \right\|_\text{trace}
    + \left\| \mM_i^{(0, k+1)} - \mC_i^{(0)} \right\|_\text{trace}.
\end{align*}
Then, using the following inequality:
\begin{align*}
    \mathbb{E} \left\| \nabla f (\mX^{(0, 0)}) - \mC^{(0)} \right\|_\star \leq \frac{\rho}{n} \sum_{i=1}^n \sqrt{\mathbb{E} \left\|  \nabla f_i (\mX^{(0, 0)}) - \mC^{(0)}_i \right\|^2_F} \leq \rho \sigma_0,
\end{align*}
we obtain the desired result.
\end{proof}

\begin{lemma}
\label{lemma:convergence_rate_with_inexact_lmo}
Suppose that \cref{assumption:smoothness,assumption:stochastic_noise} hold,  $\mC_i^{(0)} \coloneqq \mM_i^{(0, 0)}$ and $\mC^{(0)} \coloneqq \tfrac{1}{n} \sum_{i=1}^n \mC_i^{(0)}$, there exists $\eta$ and $\alpha$ such that we have
\begin{align*}
    \frac{1}{R K} \sum_{r=0}^{R-1} \sum_{k=0}^{K-1} \mathbb{E} \left\| \nabla f (\mX^{(r, k)}) \right\|_p
    &\leq \mathcal{O} \left( 
        \left( \frac{L r_0 \rho^2 \sigma^2}{S R K} \right)^\frac{1}{4}
        + \left( \left( \frac{n}{S} \right)^2 \frac{L r_0 \rho \sigma}{R \sqrt{K}} \right)^\frac{1}{3}
        + \left( \frac{L r_0}{R} \left( \frac{n}{S}  \right)^2 \right)^\frac{1}{2}
        + \frac{\rho \sigma_0}{R} \left( \frac{n}{S} \right) \right. \\
        &\quad \left. + \rho \sigma_0 \left( \frac{\rho^2 \sigma^2 K S}{L r_0 R n^2} \right)^{\frac{1}{2}}
        + \rho \sigma_0 \left( \left( \frac{n}{S} \right) \frac{\rho^2 \sigma^2 K^2}{L r_0 R^2}\right)^\frac{1}{3}
    \right),
\end{align*}
Then, $p$ is defined as follows:
\begin{align*}
    p &\coloneqq 1 + \frac{\log \left( 1 - (1 - \kappa)^{{1.5}^T} \right)}{\log \kappa}, \\
    \kappa &\coloneqq \min_{j,i,r,k} \frac{s_{j,i,r,k}}{\sqrt{\sum_{j'} s_{j',i,r,k}^2}} \; (> 0),
\end{align*}
where $\{ s_{j,i,r,k} \}_j$ are non-zero singular values of $\mM_i^{(r, k+1)} - \mC_i^{(r)} + \mC^{(r)}$. 
\end{lemma}
\begin{proof}
Even if we solve the LMO approximately, the statements of \cref{lemma:approximation_of_global_function,lemma:difference_of_momentum} hold.
Thus, combining \cref{lemma:descent_lemma_for_muon,lemma:approximation_of_global_function,lemma:difference_of_momentum} and tuning the hyperparameters as in \cref{lemma:convergence_rate}, we obtain the desired result.
\end{proof}

\newpage
\section{Hyperparameter Tuning Strategy}
\label{sec:hyperparameter_tuning}

In our experiments, the hyperparameters were tuned individually for each combination of method, dataset, and random seed.

\begin{table}[h]
    \centering
    \caption{Hyperparameter tuning strategy for each method.}
    \begin{tabular}{lcl}
    \toprule
            FedAvg                   & Stepsize                 & Grid search over $\{ 0.1, 0.01, 0.001 \}$ \\
            FedAvg (Adam)            & Stepsize                 & Grid search over $\{ 0.1, 0.01, 0.001 \}$ \\
            SCAFFOLD                 & Stepsize                 & Grid search over $\{ 0.1, 0.01, 0.001 \}$ \\
            SCAFFOLD (Adam)          & Stepsize                 & Grid search over $\{ 0.1, 0.01, 0.001 \}$ \\
            \multirow{2}{*}{LocalMuon} & Stepsize of Muon         & Grid search over $\{ 0.001, 0.0001 \}$ \\
                                     & Stepsize of Momentum SGD & Grid search over $\{ 0.1, 0.01 \}$ \\
            \multirow{2}{*}{FedMuon} & Stepsize of Muon         & Grid search over $\{ 0.001, 0.0001 \}$ \\
                                     & Stepsize of Momentum SGD & Grid search over $\{ 0.1, 0.01 \}$ \\
    \bottomrule
    \end{tabular}
\end{table}

In the following tables, we list the hyperparameters tuned by the grid search.
The reported hyperparameters correspond to the values selected from two independent trials with different random seeds.

\begin{table}[h]
\centering
\caption{Hyperparameters tuned for FashionMNIST.}
\begin{tabular}{lcc}
\toprule
                & $\beta=10.0$                         & $\beta=0.1$                          \\
\midrule
FedAvg          & $\{ 0.1, 0.1 \}$                     & $\{ 0.1, 0.1 \}$                     \\
FedAvg (Adam)   & $\{ 0.1, 0.1 \}$                     & $\{ 0.01, 0.01\}$                    \\
SCAFFOLD        & $\{ 0.1, 0.1 \}$                     & $\{ 0.1, 0.1\}$                      \\
SCAFFOLD (Adam) & $\{ 0.001, 0.001 \}$                 & $\{ 0.001, 0.001\}$                  \\
LocalMuon       & $\{ (0.001, 0.1), (0.001, 0.01) \}$  & $\{ (0.001, 0.1), (0.001, 0.1) \}$   \\
FedMuon         & $\{ (0.001, 0.01), (0.001, 0.01) \}$ & $\{ (0.001, 0.01), (0.001, 0.01) \}$ \\
\bottomrule
\end{tabular}
\end{table}

\begin{table}[h]
\caption{Hyperparameters tuned for CIFAR-10.}
\centering
\begin{tabular}{lcc}
\toprule
                & $\beta=10.0$                         & $\beta=0.1$                          \\
\midrule
FedAvg          & $\{ 0.1, 0.1\}$                      & $\{ 0.1, 0.1 \}$                     \\
FedAvg (Adam)   & $\{ 0.001, 0.001\}$                  & $\{ 0.01, 0.001\}$                   \\
SCAFFOLD        & $\{ 0.1, 0.1\}$                      & $\{ 0.1, 0.1\}$                      \\
SCAFFOLD (Adam) & $\{ 0.001, 0.001\}$                  & $\{ 0.001, 0.001\}$                  \\
LocalMuon       & $\{ (0.001, 0.01), (0.001, 0.01) \}$ & $\{ (0.001, 0.01), (0.001, 0.01)\}$  \\
FedMuon         & $\{ (0.001, 0.01), (0.001, 0.01) \}$ & $\{ (0.001, 0.01), (0.001, 0.01) \}$ \\
\bottomrule
\end{tabular}
\end{table}

\end{document}